\documentclass[sigconf,nonacm]{acmart}
\usepackage{bm}
\usepackage[inline]{enumitem}
\usepackage{subcaption}
\microtypesetup{expansion=false,protrusion=false}
\usepackage{xspace}
\usepackage{multirow}
\usepackage{amsmath}
\usepackage{booktabs}
\usepackage{tabularx}
\usepackage{adjustbox}
\usepackage{float}
\usepackage{graphicx}
\usepackage{caption}
\usepackage{textcomp}
\usepackage{hyphenat}
\usepackage{newunicodechar}
\newunicodechar{Δ}{\ensuremath{\Delta}}
\usepackage{tikz}
\usetikzlibrary{arrows.meta,positioning,fit,calc,decorations.pathreplacing}

\newtheorem{theorem}{Theorem}
\newtheorem{proposition}{Proposition}
\newtheorem{lemma}{Lemma}
\newtheorem{corollary}{Corollary}
\newtheorem{definition}{Definition}

\AtBeginDocument{%
  }

\begin{document}
\raggedbottom


\title{Decoupled Learning and Selection in Slate Recommendation for Privacy and Stability Under Noisy Scores}
\renewcommand{\shorttitle}{Privacy and Stability Under Noisy Scores}

\begin{abstract}
We formalize slate recommendation as a randomized score learner followed by deterministic selection. First, an appropriately scoped differential-privacy guarantee passes through selection and its audit trace by post-processing. End-to-end privacy holds only when selector inputs are public or independent, previous private outputs, or separately privacy-accounted; fixing raw state or candidate information instead yields only a conditional guarantee. Second, we derive a logged margin certificate: bounded score-induced objective movement below half the smallest greedy decision margin guarantees that the ordered slate is unchanged.

Controlled fixed-margin tests show near-linear exponent scaling, with an empirical slope of $-0.220$ (95\% CI $[-0.231,-0.210]$) against the independent-noise reference $-1/4$. Real-anchor experiments on OULAD, MovieLens-25M, and Amazon Musical Instruments show that greater anchor weight reduces score-noise-induced ranking churn. OULAD and EdNet certificate checks validate the implementation of the logged inequality, while closed-loop simulations show bounded target drift and setting-dependent downstream utility. The contribution is therefore a privacy-scope contract and a certifiable score-to-slate stability mechanism, not a universal utility claim.
\end{abstract}

\keywords{slate recommendation, re-ranking, differential privacy, post-processing, auditability, exploration, diversification, governance}


\author{Sam Urmian}
\orcid{0000-0001-8106-2198}
\affiliation{%
  \institution{University of Bergen}
  \department{Centre for the Science of Learning \& Technology (SLATE)}
  \city{Bergen}
  \country{Norway}}
\email{sam.urmian@uib.no}

\author{Qinyi Liu}
\orcid{0009-0003-4973-0901}
\affiliation{%
  \institution{City University of Macau}
  \department{School of Education}
  \city{Macau}
  \country{China}}
\email{qyliu@cityu.edu.mo}

\author{Mohammad Khalil}
\orcid{0000-0002-6860-4404}
\affiliation{%
  \institution{University of Bergen}
  \department{Centre for the Science of Learning \& Technology (SLATE)}
  \city{Bergen}
  \country{Norway}}
\email{mohammad.khalil@uib.no}

\renewcommand{\shortauthors}{Urmian et al.}

\maketitle

\section{Introduction}
\label{sec:intro}

When a recommender shows a list of items, the list is usually not just the direct output of a learned model. A model may first score candidate items, while a second decision layer removes restricted items, balances content, adds diversity, applies operational rules, and chooses the final list length. This second layer is often deterministic: given the same scores and inputs, it will produce the same slate. Yet guarantees for the learned model---differential privacy~\cite{McSherryMironov2009,Abadi2016}, contextual-bandit regret~\cite{Li2010,Chu2011}, fairness, or calibration~\cite{Steck2018,SinghJoachims2018,Geyik2019}---do not automatically describe this rule layer.

This separation matters because different questions attach to different parts of the pipeline. Privacy asks what the training data can reveal. Auditability asks whether a recommendation can be replayed from logs. Stability asks whether small score changes can alter the final slate. Treating the whole recommender as one learned object makes these questions harder to separate, and can make guarantees look broader than their assumptions allow.

We therefore study the two-step setting directly. A randomized learning operator $\mathcal{L}$ produces item scores, and a deterministic selection operator $\mathcal{S}$ maps those scores, state, and bounded side inputs to a slate. The paper is not a proposal for a new recommender library or a claim that this structure is universally optimal. It asks a narrower question:

\begin{quote}
\emph{If a slate recommender is constrained to learn scores first and then choose slates deterministically, what privacy, auditability, drift, and stability guarantees follow, and under which assumptions?}
\end{quote}

Section~\ref{sec:framework} makes this model class precise. The key distinction is between an \emph{end-to-end private regime}, where selector inputs are public or independent, previous DP outputs, or separately privacy-accounted, and an \emph{anchored regime}, where the selector may also consume a stable auxiliary score. Fixing raw state or candidate information yields only a conditional privacy statement. A non-private anchor can support stability analysis, but it does not preserve end-to-end DP.

The main technical object is a margin-based stability certificate. For deterministic selectors that build a slate through a finite sequence of score-dependent choices, positive stepwise margins certify that bounded score-induced perturbations cannot change the ordered slate. The certificate is auditable because margins and perturbation envelopes are logged decision-trace quantities. It specializes to a raw top-$K$ flip bound under pairwise score-difference noise and to a final-slate certificate for the greedy shaped/diversity-aware selector. These are stability statements, not privacy statements. Privacy inheritance is handled separately by standard post-processing and group-privacy arguments under the admissibility contract.

Our contributions are:
\begin{enumerate*}[label=(\roman*)]
  \item a privacy-scope contract for deterministic slate post-processing and audit traces;
  \item a logged margin certificate for invariance of an ordered greedy slate; and
  \item mechanism-focused tests across four datasets, plus scoped closed-loop simulations that separate score stability from downstream utility.
\end{enumerate*}

The closed-loop testbed is adaptive practice: items are practice questions, $d(i)\in[0,1]$ is a difficulty proxy, and $\Delta_t$ implements pacing. OULAD and EdNet provide contrasting interaction density, while MovieLens-25M and Amazon Musical Instruments test whether the raw score-stability behavior appears outside education.

\noindent\textbf{Reproducibility.}
Diagnostic code, experiment artifacts, and full proof details are available at the public repository
\url{https://github.com/mlgorithm/decoupled-slate-recommendation-recsys-2026}~\cite{Urmian2026Artifacts}.

\section{Setting and Model Class}
\label{sec:framework}

We study slate recommendation through a simple separation: one component \emph{learns} scores from data, and a second component \emph{selects} the slate shown to the user. The separation is analytical rather than product-specific. It lets us ask which properties come from the learner, which come from the selector, and which require assumptions about both.

\subsection{Setting}
\label{subsec:setting}

At each round, a policy receives a candidate pool of eligible items and chooses a slate to display. The user then produces round-level feedback: whether any shown item was attempted and, conditional on an attempt, whether the response was successful. Candidate features, target-value proxies such as difficulty or freshness, similarities, and uncertainty estimates are treated as selector inputs when available.

The protected unit for privacy is one exposure round: a user, the shown slate, the round index, and the observed feedback. If an implementation expands a slate into item rows for optimization, those rows are still treated as one protected round record. Record-neighboring datasets differ in one such exposure round; user-neighboring datasets differ in all exposure rounds belonging to one user.

A randomized learner $\mathcal{L}$ is $(\varepsilon,\delta)$-differentially private (DP) if, for all record-neighboring $D,D'$ and events $E$, $\Pr[\mathcal{L}(D)\in E]\le e^\varepsilon\Pr[\mathcal{L}(D')\in E]+\delta$~\cite{Dwork2014}. DP-SGD implements this guarantee for gradient training by clipping each per-example gradient to norm $C$, adding calibrated Gaussian noise, and accounting for privacy loss across updates~\cite{Abadi2016}.

\begin{definition}[Privacy status of selector-side inputs]
\label{def:admissible}
A selector-side input has one of four privacy statuses relative to protected data $D$: (i) public or independent of $D$; (ii) a previous DP output or deterministic post-processing thereof; (iii) a separately DP-accounted output, whose budget must be composed; or (iv) an explicitly conditioned input. The first three can support an end-to-end claim with the applicable composed budget. The fourth supports only a statement conditional on the fixed input and does not protect that input itself.
\end{definition}

\subsection{Two Operators: Learning and Selection}
\label{subsec:two-ops}

The model class has two operators.

\paragraph{Learning operator $\mathcal{L}$.}
The learner maps logged interactions to item scores and may be statistical, factorization-based, tree-based, neural, or bandit-based. The interface is architecture-independent, but its guarantees are not automatic: privacy inheritance requires a verified DP learner, while stability requires bounded score perturbations or the stated pairwise-tail condition. When exploration is used, $\mathcal{L}$ also exposes uncertainty widths. All recommender-controlled randomness covered by the analysis is confined to $\mathcal{L}$.

\paragraph{Selection operator $\mathcal{S}$.}
The selector is deterministic. Given a candidate pool, learner-derived scores, state, and bounded levers, it outputs a slate and an audit trace. Because the selector is deterministic, privacy can pass through it by post-processing, but its scope is determined by every non-learner input according to Definition~\ref{def:admissible}.
Figure~\ref{fig:dpr-two-ops} summarizes this two-operator contract.

\begin{figure}[t]
\centering
\includegraphics[width=\columnwidth]{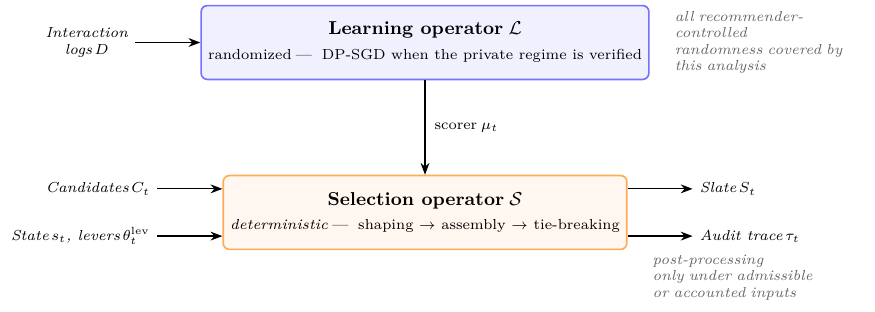}
\caption{The two-operator decomposition studied in this paper.
  $\mathcal{L}$ is the only source of learner randomness; $\mathcal{S}$ is deterministic
  given its inputs. Privacy passes to the full pipeline, including the audit trace,
  by post-processing only when the non-learner inputs satisfy the admissibility or
  DP-accounting conditions stated in Section~\ref{subsec:admissible}.}
\Description{Interaction logs enter a randomized learning operator, which outputs item scores to a deterministic selection operator. The selector also receives candidates, state, and bounded levers, and returns a slate plus an audit trace. A note limits the randomness claim to recommender-controlled randomness covered by the analysis, and another notes that post-processing privacy requires public, independent, prior-private, or separately accounted inputs.}
\label{fig:dpr-two-ops}
\end{figure}

This distinction creates three scopes. In the \emph{end-to-end private regime}, the selector consumes learner-derived scores plus only inputs of types (i)--(iii) in Definition~\ref{def:admissible}. A \emph{conditional regime} fixes inputs such as raw user state or candidate pools; the resulting guarantee does not protect those inputs. In the \emph{anchored regime}, the selector may also consume an auxiliary anchor score. A public, independent, or separately DP-trained anchor can be included in the privacy accounting. A non-private anchor trained on the protected dataset cannot: it may still support stability analysis, but the resulting pipeline is not end-to-end DP. The empirical anchor is non-private, and the closed-loop state is updated from raw simulated responses, so those experiments support stability and conditional mechanism checks rather than end-to-end privacy claims.

\subsection{Governed Policy Class}
\label{subsec:policy-class}

The governed policy class is the family of deterministic selectors we analyze. Each policy has the same public state interface and four bounded levers:
\begin{description}
\item[Target shaping.] Shifts the target value used by the proximity bonus, e.g., to study calibrated progression or difficulty targeting.
\item[Exploration.] Increases the value of items whose scores are more uncertain.
\item[Diversity.] Trades shaped relevance against similarity to items already chosen for the slate.
\item[Slate size.] Sets how many items are displayed in the round.
\end{description}

Figure~\ref{fig:dpr-three-stages} expands the default selector used in the closed-loop experiments. It first shapes scores with target and exploration terms, then greedily assembles the slate with diversity and novelty terms, and finally applies deterministic tie-breaking. The window-enforcing SCPO variant keeps the same learner, state, levers, and audit contract, but replaces the soft-window greedy assembly rule with a feasibility-constrained selector.

\begin{figure*}[t]
\centering
\includegraphics[width=0.95\textwidth]{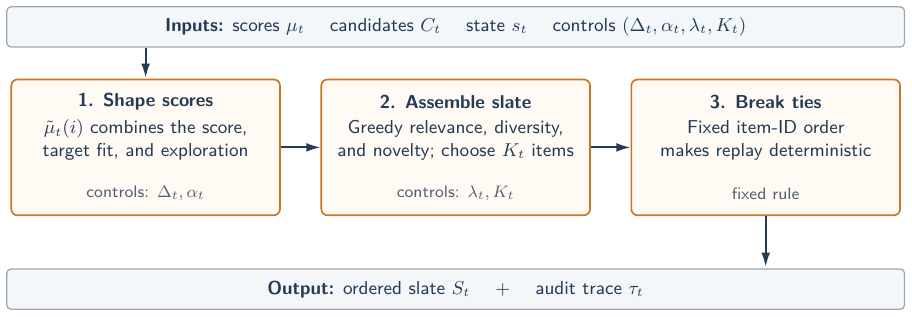}
\caption{Inside the selection operator $\mathcal{S}$: three deterministic stages.
  The shaping stage uses the score, target proximity, and exploration bonus; the
  assembly stage uses diversity, slate size, and novelty memory; tie-breaking is
  fixed so that audit replay is reproducible.}
\Description{Three-stage diagram. Inputs are scores, candidates, state, and selector controls. Shape scores uses delta and alpha, assemble slate uses lambda and K, and a fixed item-ID rule breaks ties. The output is an ordered slate and audit trace.}
\label{fig:dpr-three-stages}
\end{figure*}

The guarantees remain scoped. End-to-end privacy requires inputs of types (i)--(iii); type-(iv) inputs yield only a conditional statement. Drift requires window enforcement, or a logged soft-window event where selected items satisfy the window. Final-slate stability requires the margin certificate summarized in Section~\ref{sec:theory}.

\subsection{Auditability}
\label{subsec:audit}

The selector is replayable because its inputs and tie-breaking rule are fixed. For each round, the audit trace must contain, or allow deterministic lookup of, the candidate identities, levers, target value, state, novelty memory, learner-derived scores, shaped scores, target proxies, uncertainty widths, pairwise similarities, final slate, and tie-breaking rule. Given those inputs, replaying the selector reproduces the slate exactly.

This trace has two uses. First, it lets us attribute a decision either to the learned score or to an explicit selector lever. Second, it gives the quantities used by the stability certificate: margins and perturbation envelopes are decision-trace quantities, not hidden proof artifacts. The trace is private only under the same admissibility/accounting conditions as the selector itself.

At production scale, exact replay needs $O(|C_t|)$ scalar fields and at most $O(K_t|C_t|)$ queried similarities, not a dense $O(|C_t|^2)$ matrix; certificate-only monitoring stores $O(K_t)$ margins, $M_t$, and versioned input references. Selection already computes these values, so logging adds serialization and storage but no new asymptotic ranking pass. Both trace levels inherit the privacy requirements above.

\subsection{What This Model Class Buys}
\label{subsec:framework-whatbuys}

The two-operator restriction has four consequences, summarized in Section~\ref{sec:theory}. A private learner passes its privacy guarantee through deterministic selection when the side-information contract is satisfied. Record-level privacy can be lifted to user-level privacy by the standard group-privacy conversion. A window-enforcing selector gives a scorer-independent drift envelope. Finally, anchor blending can reduce raw top-$K$ score instability, and the selector-margin certificate explains when that score stability reaches the final greedy slate.

Sections~\ref{sec:methodology}--\ref{sec:results} evaluate these consequences with closed-loop runs on two educational corpora with contrasting user--item density, raw score-stability checks on two ratings corpora, and component-removal runs that attribute outcome changes to individual selector components.

\section{Related Work}
\label{sec:related}

This paper connects several lines of recommender-systems work that are usually analyzed separately. Differentially private recommendation protects the learned model or training procedure, from private collaborative filtering to DP-SGD-based recommenders~\cite{McSherryMironov2009,Abadi2016,Dwork2014,Liu2025DP}. Those results do not by themselves explain what happens after a private scorer is passed through top-$K$ selection, re-ranking, slate-size adaptation, or an audit log. Our privacy-inheritance result uses standard post-processing and group-privacy tools, but distinguishes end-to-end protection from a statement that merely conditions on raw selector inputs.

Deterministic re-ranking and slate construction methods, including MMR, calibrated recommendation, diversity, and fair exposure re-rankers~\cite{CarbonellGoldstein1998,KunaverPozrl2017,Steck2018,SinghJoachims2018,Geyik2019}, are the closest structural relatives of our selector. They usually study the re-ranking objective itself, while privacy and scorer-noise stability are outside the main analysis. Contextual bandits and private bandits~\cite{Li2010,Chu2011,Abbasi2011,Shariff2018,Mishra2015} provide exploration and regret guarantees, but couple exploration to a specific learner; our LinUCB result is therefore only a scoped recovery statement when the other selector terms are disabled. Safe and constrained RL methods for recommendation~\cite{Ie2019SlateQ,Achiam2017CPO} enforce constraints inside learned policies, whereas our drift certificate is a selector-side statement that can be checked without retraining the scorer. Off-policy learning and evaluation~\cite{Swaminathan2015CRM} are complementary ways to train $\mathcal{L}$ from biased logs rather than guarantees about the deterministic post-layer.

Adjacent stability work asks related but different questions. Rank-list sensitivity measures how recommendation lists react to training-interaction perturbations~\cite{Oh2022RankSensitivity}, while our object is a logged sufficient invariance certificate at the selector decision. Stability mappings for rankings under uncertain predictions study stable and fair ranking rules~\cite{Devic2024Stability}; we instead bound objective movement through deterministic greedy slate construction. Stability-based private top-$K$ selection uses stability to decide when a set can be released privately~\cite{ZhuWang2022PrivateK}; our direction is reversed, propagating an already-private learner through a deterministic selector under an explicit side-information contract.

The gap addressed here is the structural separation itself. In the private regime, privacy and audit replay follow from standard composition once the side-information contract is satisfied. In the anchored regime, a non-private auxiliary anchor gives no end-to-end DP guarantee, but it enables a separate score-stability analysis. The main technical contribution is therefore not a new scorer or a new re-ranking objective; it is the margin-certified link between score perturbations, top-$K$ ranking stability, and greedy-slate stability under named assumptions.

\section{Theoretical Guarantees}
\label{sec:theory}

This section states the guarantees and explains the intuition behind them. The two questions are separate. First, when does the full recommender inherit privacy from its learner? Second, when can we certify that score noise does \emph{not} change the slate that users see? End-to-end privacy is inherited only when every selector-side input is public or independent, a previous DP output, or separately privacy-accounted. Fixing raw side information yields only a conditional statement. Stability can also be studied in the anchored regime, but then it is a stability claim rather than an end-to-end privacy claim.

The main stability idea is simple. A deterministic selector is stable when the chosen item clearly beats the closest alternative. We call that separation the \emph{margin}. Score noise can move the selector objective; we call the largest such change the \emph{movement}. If the movement is less than half the margin, no runner-up can overtake the selected item. This is the core certificate behind the top-$K$ and greedy-slate results.

\begin{table}[t]
\caption{Reading the stability quantities.}
\label{tab:stability-quantities}
\small
\begin{tabular}{@{}p{0.18\columnwidth}p{0.76\columnwidth}@{}}
\toprule
Quantity & Plain meaning \\
\midrule
$\Gamma$ or $\Gamma_t$ & The smallest logged decision margin: how clearly the chosen item beat the closest runner-up. \\
$M$ or $M_t$ & The largest score-induced movement in the selector objective. \\
$M<\Gamma/2$ & A certificate that the ordered slate is unchanged. \\
$\gamma$ & The gap between anchor top-$K$ items and the remaining candidates. \\
$w$ & Anchor weight; larger $w$ gives the noisy adaptive score less influence. \\
\bottomrule
\end{tabular}
\end{table}

\subsection{Privacy Scope}
\label{subsec:admissible}

The privacy statements are privacy-inheritance results, not new DP mechanisms. If the learner $\mathcal{L}$ is DP, then any deterministic re-ranking of its output is post-processing and spends no additional privacy budget. The catch is that the selector may also use candidate pools, state, similarities, levers, novelty memory, or anchors. Definition~\ref{def:admissible} distinguishes inputs covered by an end-to-end claim from inputs that are merely fixed by conditioning.

The post-processing result formalizes this statement for our two-operator pipeline. The same logic applies to the audit trace: replay logs are private only when every non-learner entry is public or independent, derived from a previous DP output, or separately accounted. A non-private anchor trained on $D$ is therefore outside the private-regime claim, even if the downstream selector is deterministic. If each user contributes at most $B$ protected round records, record-level $(\varepsilon,\delta)$-DP lifts by group privacy to
\[
\left(B\varepsilon,\ \delta\sum_{j=0}^{B-1}e^{j\varepsilon}\right)\text{-DP}.
\]
The experimental protocol does not impose a uniform contribution cap $B$, so we do not report a single user-level budget; such a claim would require clipping each user's contribution first.

\subsection{Selector Drift}
\label{subsec:lem-drift}

The drift guarantee is a selector-side result. It does not depend on whether the scorer is neural, matrix-factorized, private, or non-private. If the selector enforces a window around the current target value, then the average target value of accepted items cannot jump arbitrarily from one round to the next. The envelope is
\[
|\bar d(A_{t+1})-\bar d(A_t)| \le \eta+2\Delta_{\max}+w_t+w_{t+1}.
\]

The meaning is direct: $\eta$ captures how fast the state estimate can move, $\Delta_{\max}$ caps how far the target lever can shift the target, and $w_t,w_{t+1}$ are the two window widths. The drift result is exact for a window-enforcing selector such as SCPO. For the default soft-window MMR selector, it applies only on logged rounds where the selected items actually satisfy the window condition.

\subsection{Margin-Certified Stability}
\label{subsec:thm-selector-stability}

\begin{theorem}[Margin-certified selector stability]
\label{thm:selector-stability-main}
Let a deterministic selector build an ordered slate in $K$ greedy steps using objective $F_q^\mu(i\mid S_{q-1})$. Under reference scores $\mu^0$, let $S_{q-1}^0$ be the prefix and $i_q^0$ the selected item at step $q$. Define
$\Gamma_q=F_q^{\mu^0}(i_q^0\mid S_{q-1}^0)-\max_{j\ne i_q^0}F_q^{\mu^0}(j\mid S_{q-1}^0)$ over remaining candidates, and $\Gamma=\min_q\Gamma_q$. Suppose $|F_q^\mu(i\mid S_{q-1}^0)-F_q^{\mu^0}(i\mid S_{q-1}^0)|\le M$ for every remaining $i$ along the reference trajectory. If $\Gamma>0$ and $M<\Gamma/2$, then $\mu$ and $\mu^0$ produce the same ordered slate.
\end{theorem}

\begin{proof}[Proof sketch]
Each candidate objective moves by at most $M$. Hence a reference gap can shrink by at most $2M$. Since $\Gamma-2M>0$, induction preserves every greedy choice and the ordered slate.
\end{proof}

The logged $\Gamma$ and a valid envelope $M$ make the certificate auditable; failure means non-certification, not necessarily a slate change. The supplement gives the full proof and corollaries.

\subsection{What Anchor Blending Buys}
\label{subsec:cor-topk}

For raw top-$K$ selection, the margin has an especially simple form: the anchor gap $\gamma$ between the original top-$K$ items and the remaining candidates. A top-$K$ flip can only happen if score noise pushes an outside item across that gap. The top-$K$ stability result turns this into an exponential bound. In the pairwise-noise form, the flip probability scales like
\[
K(|C|-K)\exp\!\left[-\frac{\gamma^2}{2(1-w)^2\tau^2}\right],
\]
where $\tau^2$ is the score-difference noise scale. Under independent coordinate noise, this reduces to the familiar $-\gamma^2/[4(1-w)^2\sigma^2]$ exponent.

The interpretation is the main point. Larger anchor gaps make the original top-$K$ harder to disturb. Larger anchor weight $w$ gives the noisy adaptive score less influence. Larger score-difference noise makes flips more likely. The biased-anchor version records that an imperfect anchor consumes part of the usable gap.

\subsection{From Raw Scores to Greedy Slates}
\label{subsec:selector-bridge}

The deployed selector is not raw top-$K$. It also uses target shaping, exploration, diversity, slate size, novelty memory, and deterministic tie-breaking. The selector-margin bridge connects score stability to final-slate stability for this greedy selector.

The bridge says to account for every way a score perturbation can move the greedy objective: score movement, bonus movement, and novelty/state movement. At a single round with fixed candidate identities, eligibility, state, bonus inputs, and novelty memory, the latter difference terms vanish. The implementation check uses the realized perturbation to compute the ex post envelope $M_t=(1-\lambda_t)\|\mu-\mu^0\|_\infty$; it is therefore a soundness check of the logged inequality, not a prospective probabilistic certificate. When chaining across rounds or allowing other inputs to change, their movement must also be bounded. The same half-margin rule then applies.

The top-$K$ and greedy certificates answer related but different questions. The top-$K$ result is a clean raw-ranking statement and is the easiest place to see the exponential anchor effect. The greedy result is closer to the actual slate users see, but it requires the richer decision trace and the observed stepwise margins. The high-probability greedy version follows when the score errors are sub-Gaussian.

\subsection{Scoped Recovery and Threat Model}
\label{subsec:prop-regret}

The LinUCB recovery result is a sanity check: if target shaping, diversity, novelty, and slate-size adaptation are disabled, the selector reduces to the standard LinUCB rule and recovers the usual regret rate. It is not a regret theorem for the full active-lever selector.

The threat model is therefore deliberately split. Privacy protects the learner output and deterministic post-processing only under the admissible-side-information contract. It does not hide on-screen exposures or raw user responses. Stability is independent of privacy: it can hold for a non-private anchored system, and it composes with DP only when the anchor is public, independent of $D$, or separately privacy-accounted.
\label{subsec:scope}

\section{Experimental Methodology}
\label{sec:methodology}

We use the experiments as mechanism checks rather than as a product benchmark. The closed-loop setting asks what happens when the learner--selector separation is imposed on adaptive practice recommendation; the raw score-ranking experiments test the stability theorem directly. The educational corpora are OULAD~\cite{Kuzilek2017} and EdNet~\cite{Choi2020}. OULAD is relatively dense (23{,}351 users, 188 items, 173{,}739 interactions, $39.58$ interactions per thousand possible user--item pairs), while EdNet is sparse but longer-horizon (296{,}701 users, 11{,}555 items, 23.4M interactions, $6.82$ interactions per thousand possible pairs). To test whether the score-stability pattern is education-specific, we also run raw ranking checks on MovieLens-25M~\cite{HarperK16} and SNAP Amazon Musical Instruments reviews.

\subsection{Closed-loop recommendation setting}
\label{subsec:method-overview}

At each round $t$, a policy sees a matched candidate pool $C_t$ (size 100 unless stated), returns a slate $S_t$, receives simulated item-level responses that are aggregated into the round feedback used by the state update, and then advances to the next round. Candidate pools and random seeds are matched across modes, so paired differences are attributable to the learner/selector choice rather than retrieval. Sessions run for $T=20$ rounds, with one $T=50$ longer-horizon variant. OULAD and EdNet do not log exposure sets or shown-but-skipped items, although impression-aware corpora exist in other domains, including MIND, ContentWise Impressions, and FINN.no Slates~\cite{PerezMaurera2025}. We therefore measure engagement and utility through a response simulator calibrated against item correctness, user correctness, and interaction-count aggregates; these checks support trend-level comparisons, not absolute field-effect claims.

\subsection{Modes and privacy}
\label{subsec:method-design}

The main modes separate scorer quality from selector behavior. \textsc{Adaptive} uses a neural scorer that maps the current learner-state representation and candidate-item features to $\mu_t^{\mathrm{neural}}(i)$, updates from observed round feedback, and feeds those scores to the deterministic selector. \textsc{ALS+Selector} uses the same selector with a fixed ALS scorer, isolating the post-layer. \textsc{Adaptive+ALS} blends a neural adaptive scorer with an ALS anchor,
\[
  \mu_t^{w}(i)=w\,\mu_t^{\mathrm{ALS}}(i)+(1-w)\,\mu_t^{\mathrm{neural}}(i),
\]
with anchor weight $w=0.75$ in the main anchored runs and additional $w$ sweeps in the diagnostics. For every direct stability diagnostic, the ALS and neural score vectors are separately standardized to zero mean and unit standard deviation within each candidate pool before blending; thus $w$ is a convex weight on a common within-pool orientation and scale. The released closed-loop artifacts do not record an equivalent scaling transform, so closed-loop values of $w$ are treated as configuration labels rather than cross-model calibration parameters. We use $w$ for anchor weight so it is not confused with the selector diversity lever $\lambda_t$ in Section~\ref{sec:framework}. Since the ALS anchor is trained non-privately on the protected data, this mode is an anchored-regime stability experiment, not an end-to-end DP pipeline. \textsc{Adaptive-SCPO} replaces soft-window greedy selection with a window-enforcing selector for the exact drift certificate. Baselines include \textsc{Fixed}, \textsc{Greedy-Target}, \textsc{LinUCB}, pure \textsc{ALS}, and three educational state trackers (BKT, DKT, SAKT) used only as scorer-interface comparators.

The experimental harness records three nominal privacy configurations: standard ($C=1,\sigma=1.2,q=0.020,\delta=10^{-5}$), strong ($C=1,\sigma=2.0,q=0.015,\delta=10^{-5}$), and locked ($C=1,\sigma=3.0,q=0.010,\delta=10^{-5}$). The associated accountant labels at 20 updates are approximately $0.36$, $0.16$, and $0.07$. The released artifact, however, does not establish the sampling mechanism, population denominator behind $q$, or per-round gradient aggregation before clipping. We therefore report these only as configured-update stress-test labels, not independently verified $(\varepsilon,\delta)$ guarantees. The formal privacy inheritance result applies when the learner implementation clips at the declared protected unit, adds calibrated noise under a specified sampler/accountant, and accounts for every adaptive update and release.

\subsection{Direct theorem checks}
\label{subsec:method-setup}

The main stability theorem is tested outside the simulator first. We construct fixed-margin pools with $|C|=100$, $K=10$, controlled top-vs-rest gaps, anchor weights $w\in\{0,0.50,0.75\}$, and Gaussian score perturbations. For independent Gaussian noise, the fitted slope of
\[
  \log(\widehat p_{\mathrm{flip}}/[K(|C|-K)])
  \quad\text{against}\quad
  \gamma^2/((1-w)^2\sigma^2)
\]
tests the $-1/4$ exponent in the top-$K$ corollary. Each cell contains 64 pools $\times$ 500 perturbations, or 32{,}000 trials. The descriptive OLS fit uses exactly the eight cells whose smoothed Flip@$K$ lies in $[10^{-4},0.25]$; this analysis rule was not preregistered. Its slope standard error is 0.0042 and its two-sided 95\% $t$ interval is $[-0.231,-0.210]$. A low-rank correlated-noise version fits the generalized pairwise form using the empirical $\hat\tau^2$ scale. We then repeat real-anchor Flip@$K$ sweeps on OULAD, MovieLens-25M, and Amazon Musical Instruments, using independent and low-rank correlated perturbations. The ratings datasets are binarized at rating $\ge 4$; threshold sensitivity checks preserve the same qualitative pattern.

Two diagnostics connect the theorem to the implemented selector. The pairwise-tail diagnostic samples cross-boundary item pairs, normalizes score differences by empirical $\hat\tau$, and checks the sub-Gaussian upper-tail envelope. The selector-margin diagnostic fixes candidate identity and eligibility, state, bonus inputs, novelty memory, and selector parameters; after drawing a realized Gaussian score perturbation $\xi$, it computes $M_t=(1-\lambda_t)\|\xi\|_\infty$ and checks whether any case satisfying $M_t<\Gamma_t/2$ changes the slate. It is therefore an ex post certificate-implementation check. A separate 200-seed learner-update diagnostic on fixed OULAD pools is descriptive only: its 160{,}000 pair draws share training seeds, pools, and item scores and are not independent inferential observations.

\paragraph{Reproducibility details.}
The headline closed-loop summaries contain 256 matched learner-sessions per mode across four response models and four user profiles. Data are split by user with seed 42; candidate pools contain 100 items and are generated once per round for replay across modes. Sessions use 20 updates unless the 50-round variant is named. The adaptive scorer is a two-tower network (96 hidden and 48 embedding dimensions in the closed-loop configuration), and the fixed factorization scorer uses 32 latent dimensions and five training epochs. Simulator validation and preprocessing details are in the supplement; fixed-margin trials and real-anchor seed counts are stated above and in the released manifests.

\subsection{Outcome metrics and statistics}
\label{subsec:method-metrics}

For reference top-$K$ set $T_r^0$ and compared set $T_r$ in trial $r$, the raw ranking metrics are
\[
\mathrm{Flip@}K=\frac{1}{N}\sum_{r=1}^N\mathbb{I}\{T_r\ne T_r^0\},\qquad
\mathrm{Jaccard@}K=\frac{1}{N}\sum_{r=1}^N\frac{|T_r\cap T_r^0|}{|T_r\cup T_r^0|}.
\]
Flip@$K$ estimates the frequency of any top-$K$ set change; Jaccard reports its magnitude when flips saturate. Closed-loop metrics are attempts per round, item-level uptake, round-level correctness, a fixed pre/post utility probe, target-value jump $|\bar d(A_{t+1})-\bar d(A_t)|$, slate overlap, and Kendall-$\tau$ rank agreement. The probe is the change in mean Rasch/IRT expected correctness over the same fixed difficulty grid $d\in[0,1]$ before and after a session; its raw range is $[-1,1]$, and reported deltas use the $\times10^{-3}$ scale. All closed-loop comparisons are paired at the session level under matched seeds and candidate pools; confidence intervals are bootstrap 95\% intervals. Sign-test $p$-values are exploratory descriptive checks without multiplicity correction and therefore do not support family-wise confirmatory claims.

\section{Results}
\label{sec:results}

The results separate direct theorem checks from downstream closed-loop behavior. The controlled fixed-margin experiment shows the predicted linear scaling, although its fitted independent-noise slope is shallower than the exact reference. Real-anchor sweeps show the same monotone pattern across one educational dataset and two ratings datasets, pairwise-tail diagnostics check the noise condition, and selector-margin gates on OULAD and EdNet serve as implementation checks of the logged certificate. The closed-loop experiments are more mixed: they ask what happens when the stability mechanism is embedded inside a simulator with target shaping, diversity, slate-size adaptation, and response feedback.

\begin{table*}[t]
\centering
\caption{Main evidence map. Each row links one claim to the diagnostic that tests it most directly; downstream utility is reported separately because it also depends on simulator dynamics and optimization.}
\label{tab:main-evidence-map}
\scriptsize
\begin{tabular}{@{}p{0.18\textwidth}p{0.29\textwidth}p{0.46\textwidth}@{}}
\toprule
Claim being tested & Diagnostic & Strongest result \\
\midrule
Top-$K$ exponent & Fixed-margin calibration under independent and low-rank correlated score noise & Independent noise: slope $-0.220$ (95\% CI $[-0.231,-0.210]$) vs. reference $-0.250$ ($R^2=0.998$). Correlated noise: slope $-0.506$ vs. reference $-0.500$ ($R^2=0.997$). \\
Anchoring reduces raw ranking churn & Real-anchor sweeps on OULAD, MovieLens-25M, and Amazon Musical Instruments & Anchor weight $w=0.75$ cuts Flip@$K$ by 53--81\% across the three corpora and tested noise scales; rating-threshold sensitivity preserves the effect. \\
Pairwise noise condition & Injected-noise tails and descriptive 200-seed learner-update diagnostic & Normalized cross-boundary score differences stay below the sub-Gaussian envelope; shared seeds, pools, and scores make the repeated pair draws descriptive rather than independent observations. \\
Score-to-slate bridge & Ex post certificate-implementation checks on OULAD and EdNet & Anchor gaps and greedy margins are positive in all 640 logged pools. With candidates, state, bonuses, and novelty fixed, the realized $M_t$ check has zero instrumentation violations across 16{,}000 trials. \\
Target-value drift & Realized-window SCPO check and soft-window event logging & SCPO uses the exact realized-window envelope; for soft-window MMR, the fixed-window event holds on more than 98\% of logged round pairs. \\
Downstream stress test & Nominal OULAD update configurations and explicit EdNet score-noise injection & The interpretable noise-injection result gives clean/noisy Kendall-$\tau$ of 0.928--0.993 anchored vs. 0.748--0.904 unanchored; utility changes remain mixed. \\
\bottomrule
\end{tabular}
\end{table*}

\subsection{Selector components and drift}
\label{subsec:results-levers}

Component-removal runs show that the selector terms affect different outcomes. Removing target shaping changes accepted-item target values and the utility probe most; removing exploration mainly changes attempts per round; removing diversity changes consecutive-slate overlap; fixing slate size changes uptake. Against pure \textsc{ALS}, the full \textsc{Adaptive} mode on OULAD improves round-level correctness by 0.121 [0.105, 0.137], attempts per round by 0.714 [0.650, 0.777], and the utility probe by $14.824\times10^{-3}$ [13.172, 16.350], with absolute target-value jump 0.105. On EdNet, the same comparison gives smaller short-horizon gains: attempts per round +0.140 [0.115, 0.168] and utility-probe change $0.032\times10^{-3}$ [$-1.062$, 1.148]. Relative to default \textsc{Adaptive}, the point-estimate SCPO probe difference is $-1.339\times10^{-3}$ on OULAD but $+1.498\times10^{-3}$ on EdNet; hence the window-enforcing variant does not uniformly trade utility for its cleaner drift certificate.

\subsection{Target-value drift}
\label{subsec:results-drift}

The target-drift row of Table~\ref{tab:main-evidence-map} summarizes the certificate scope. Observed consecutive target-value jumps stay within the configured envelope. For \textsc{Adaptive-SCPO}, the envelope is exact on non-relaxation rounds and uses the realized relaxed window otherwise. For default soft-window MMR, the fixed-window event holds on more than 98\% of round pairs, so the same bound applies on those logged events. The deterministic statement belongs to the window-enforcing selector; soft-window MMR inherits it only when the event occurs.

\subsection{Raw score-ranking stability}
\label{subsec:results-score-ranking}

Figure~\ref{fig:topk-flip-theorem} gives the direct test of the top-$K$ stability result. In constructed pools with known gap $\gamma$, the eight-cell independent-Gaussian calibration fits slope $-0.220$ (SE 0.0042; 95\% $t$ CI $[-0.231,-0.210]$) with $R^2=0.998$ against the reference $-1/4$. The interval excludes $-0.25$, so the evidence supports the predicted linear form but indicates a slightly shallower finite-sample slope. The correlated-noise calibration, fit against $\gamma^2/((1-w)^2\hat\tau_{\max}^2)$, gives slope $-0.506$ with $R^2=0.997$ against the pairwise-form reference $-1/2$.

Real-anchor sweeps show the same direction outside the synthetic setting. Under independent noise on OULAD, increasing anchor weight from 0 to 0.75 reduces Flip@$K$ from 0.167 to 0.032 at $\sigma=0.02$, from 0.372 to 0.107 at $\sigma=0.05$, and from 0.603 to 0.216 at $\sigma=0.10$. The ratings-corpus checks use the larger grid $\sigma\in\{0.05,0.10,0.20\}$: on MovieLens-25M the reductions are 0.410 to 0.114, 0.653 to 0.229, and 0.875 to 0.410, while on Amazon Musical Instruments they are 0.266 to 0.066, 0.494 to 0.128, and 0.752 to 0.262. Figure~\ref{fig:real-anchor-heatmaps} shows the same monotone pattern visually across domains. Low-rank correlated noise follows the same pattern when summarized by the pairwise $\hat\tau^2$ proxy, and rating-threshold sensitivity on the two ratings corpora preserves the effect.

\begin{figure*}[t]
  \centering
  \begin{subfigure}[t]{0.48\textwidth}
    \centering
    \includegraphics[width=\linewidth]{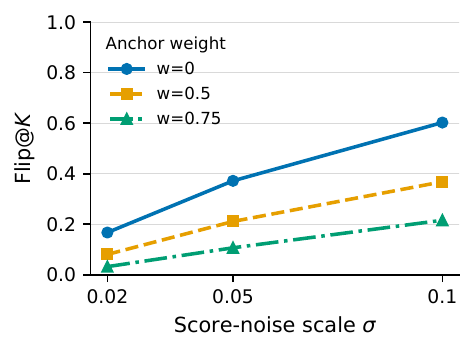}
    \caption{OULAD}
  \end{subfigure}\hfill
  \begin{subfigure}[t]{0.48\textwidth}
    \centering
    \includegraphics[width=\linewidth]{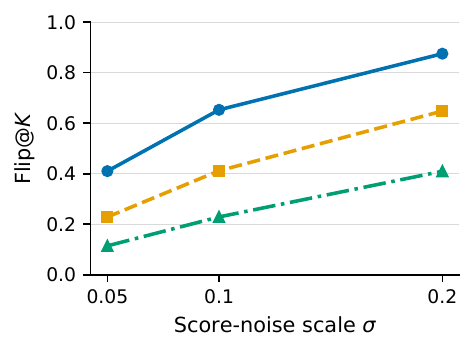}
    \caption{MovieLens-25M}
  \end{subfigure}

  \medskip
  \begin{subfigure}[t]{0.48\textwidth}
    \centering
    \includegraphics[width=\linewidth]{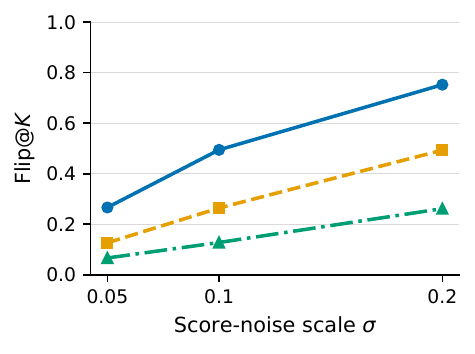}
    \caption{Amazon Musical Instruments}
  \end{subfigure}\hfill
  \begin{subfigure}[t]{0.48\textwidth}
    \centering
    \includegraphics[width=\linewidth]{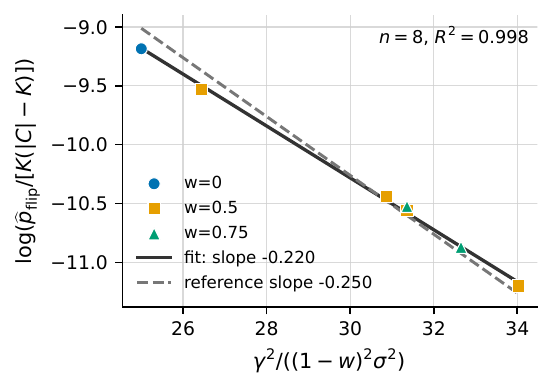}
    \caption{Fixed-margin exponent calibration}
    \label{fig:topk-flip-theorem}
  \end{subfigure}
  \caption{Direct score-stability diagnostics. Panels (a)--(c) show that raw top-$K$ churn increases with score-noise scale and decreases with anchor weight across one educational and two ratings corpora; the legend in (a) applies to (a)--(c). Panel (d) shows the eight-cell fixed-margin fit, with slope $-0.220$ (95\% CI $[-0.231,-0.210]$) versus reference $-1/4$.}
  \Description{Four panels in a two-by-two grid. OULAD, MovieLens-25M, and Amazon Musical Instruments plot the fraction of top-10 set changes against Gaussian score-noise scale; every curve rises with noise, and anchor weight 0.75 remains below 0.50 and the unanchored curve. The fourth panel plots eight fixed-margin cells against the theorem scaling variable; the fitted line is nearly linear but slightly shallower than the dashed negative-one-quarter reference.}
  \label{fig:real-anchor-heatmaps}
\end{figure*}

The pairwise-tail diagnostic confirms that the injected score-noise experiments satisfy the theorem condition: normalized cross-boundary differences have standard deviation near 1.0 and upper tails below $\exp(-u^2/2)$ for both independent and low-rank correlated noise. For example, at $u=2.0$ the measured tail is about 0.023 versus envelope 0.135. The selector-margin gate then checks the implementation of the final-slate bridge. On OULAD, $\gamma$ and $\Gamma_t$ are positive in all 320 logged candidate pools (median $\gamma=0.0577$, median $\Gamma_t=7.33\times10^{-4}$), and none of 3,877 ex post certified perturbation trials changes the greedy slate. On EdNet, the medians are smaller ($0.0071$ and $2.79\times10^{-4}$), but all 320 records again have positive margins and the certified-violation count is zero over 8,000 perturbation trials. Because $M_t$ is computed from each realized perturbation while candidates, state, bonuses, and novelty memory are held fixed, zero violations are a soundness check of the certificate instrumentation rather than independent evidence for the theorem. The full tail and margin plots appear in the supplement.

\subsection{Configured-update stress tests and closed-loop behavior}
\label{subsec:results-dp}

Across 200 learner-update seeds and 16 held-out pools, normalized score differences remain below the sub-Gaussian envelope. The 160{,}000 cross-boundary pair draws per configuration share seeds, pools, and item scores, so this is descriptive rather than inferential. An 80-update run raises $\hat\tau$ from about 0.67 to 1.17 and lowers anchor-weight-0.75 Jaccard from about 0.72 to 0.63. Because the released artifact does not establish the complete DP mechanism described in Section~\ref{subsec:method-design}, this diagnostic supports only the observed tail shape and is not evidence for a particular privacy tier or $\varepsilon$.

Closed-loop configuration comparisons are likewise treated as downstream stress tests, not verified privacy--utility curves. On OULAD, the $w=0.75$ anchored mode changes correctness by only about 0.005--0.007 across nominal tiers, compared with about 0.024--0.027 for unanchored \textsc{Adaptive}; pure \textsc{ALS} has zero change by construction because it is not updated, not because it has the best recommendation utility. On EdNet, changes are small and mixed. A separate explicit score-noise-injection run gives the interpretable stability result: clean/noisy Kendall-$\tau$ is 0.748--0.904 unanchored and 0.928--0.993 at $w=0.75$. We therefore claim score-noise attenuation, not universal downstream utility improvement.

Across four response models $\times$ four user profiles (16 settings), minimum paired deltas versus \textsc{ALS} are positive for correctness, post-test utility, and uptake on both datasets (supplementary robustness table). EdNet gains are smaller and more variable, and its configured-update utility results remain mixed. This sensitivity check affects only downstream outcomes; score- and margin-based diagnostics do not use simulated responses.

\section{Discussion}
\label{sec:discussion}

Deterministic selection separates certificates for privacy inheritance, replay, drift, and score-to-slate stability. End-to-end privacy requires public or independent inputs, prior DP outputs, or separately accounted inputs; fixing raw state yields only a conditional statement, and non-private anchors support stability rather than DP. The margin theorem is central; the other guarantees are scoped consequences.

Because OULAD and EdNet lack exposure logs, engagement is simulated. Impression logs retaining responses, candidate sets, and scores could enable real-decision $\Gamma_t/M_t$ checks. The released artifact does not establish the sampler and protected-unit clipping needed to validate its nominal privacy configurations, so those comparisons are stress tests rather than verified privacy--utility curves. A user-level guarantee additionally requires a contribution cap $B$. A failed sufficient margin certificate does not imply slate change; within these limits, the stability guarantees remain auditable.

\begin{acks}
Generative AI tools were used to assist with language editing and
manuscript revision. The authors verified all claims, analyses,
references, and final manuscript content.
\end{acks}

\onecolumn
\bibliographystyle{ACM-Reference-Format}
\bibliography{references}

\appendix

\section{Supplementary setting details}
\label{app:setting-details}

This section gives the concrete state and lever maps used by the experimental instantiation in the main paper. These rules are deterministic, public, and fixed before evaluating any policy.

\paragraph{Formal two-operator notation.}
The learning operator is a possibly randomized map $\mathcal{L}:D\mapsto(\theta_1,\ldots,\theta_T)$, with per-candidate scores $\mu_t(\cdot)=f(\theta_t,\cdot)$. The deterministic selector produces
\begin{equation}
  S_t=\mathcal{S}\!\left(C_t,\mu_t,s_t;\theta^{\mathrm{lev}}_t\right),
  \label{eq:selection}
\end{equation}
where $C_t$ is the candidate pool, $s_t$ is the pre-round state, and $\theta^{\mathrm{lev}}_t$ is the selector-lever vector. The full pipeline output is $\pi_{\mathcal{L},\mathcal{S}}(D)=(S_1,\ldots,S_T)$.

\paragraph{Formal governed selector.}
The bounded lever vector is
\[
\begin{aligned}
\theta^{\mathrm{lev}}&=(\Delta,\alpha,\lambda,K)\in\Theta_{\mathrm{gr}},\\
\Theta_{\mathrm{gr}}
&=[-\Delta_{\max},\Delta_{\max}]
\times [0,\alpha_{\max}]
\times [0,\lambda_{\max}]
\times \{K_{\min},\ldots,K_{\max}\}.
\end{aligned}
\]
The default soft-window greedy selector first computes
\[
  \tilde{\mu}_t(i)=\mu_t(i)+\eta_{\mathrm{tg}}\mathrm{prox}_w(d(i),m_t)+\alpha_t\sigma_t(i),
\]
where $\mathrm{prox}_w(d,m)=1-\min\{1,(d-m)^2/w^2\}$. Starting from $S=\emptyset$, it repeatedly adds
\[
  i^\star=\arg\max_{i\in C_t\setminus S}
  \left[(1-\lambda_t)\tilde{\mu}_t(i)
  -\lambda_t\max_{j\in S}\mathrm{sim}(i,j)
  +\nu\mathbb{I}\{i\notin H_t\}\right]
\]
until $|S|=K_t$, then applies deterministic item-id tie-breaking after fixed preprocessing and similarity computation.

\paragraph{Audit trace.}
The round trace is
\[
\begin{aligned}
\tau_t=\big(&C_t,\theta^{\mathrm{lev}}_t,m_t,s_t,H_t,
\{\mu_t(i),\tilde{\mu}_t(i),d(i),\sigma_t(i)\}_{i\in C_t},\\
&\{\mathrm{sim}(i,j)\}_{i,j\in C_t},S_t\big),
\end{aligned}
\]
together with the fixed tie-breaking rule. Public or immutable entries may be stored as identifiers plus deterministic lookup rules, but replay requires the full tuple above to be available directly or by lookup.

\paragraph{Pre-round state.}
At the beginning of round $t$, after candidate retrieval, the selector state is
\[
  s_t=(\widehat{k}_t,\widehat{e}_t,u_t^{(\mathrm{unc})},H_t).
\]
Let $a_r\in\{0,1\}$ be the attempt indicator and $y_r=a_r\tilde y_r\in\{0,1\}$ the observed correctness signal from round $r$. With neutral priors $\widehat{k}_1=\widehat{e}_1=0.5$, the response summaries for $t>1$ are
\[
\widehat{e}_t=\frac{1}{t-1}\sum_{r<t}a_r,\qquad
\widehat{k}_t=
\begin{cases}
\displaystyle\frac{\sum_{r<t}y_r}{\sum_{r<t}a_r}, & \sum_{r<t}a_r>0,\\[6pt]
0.5, & \sum_{r<t}a_r=0.
\end{cases}
\]
Uncertainty widths are normalized before use in the lever rule:
\[
u_t^{(\mathrm{unc})}=\frac{1}{|C_t|}\sum_{i\in C_t}\min\{1,\sigma_t(i)/\sigma_{\mathrm{ref}}\},
\]
where $\sigma_{\mathrm{ref}}>0$ is fixed from the training split before evaluation. The novelty memory $H_t$ is the ordered list of the last $L_H$ item identifiers displayed in $S_1,\ldots,S_{t-1}$, or all prior displayed identifiers if fewer than $L_H$ exist. Thus $s_t=\mathcal{U}(S_{<t},o_{<t},C_t)$ is a deterministic function of prior displayed slates, prior responses, and the current logged candidate pool; no hidden randomization is used in $\mathcal{U}$.

\paragraph{Lever map.}
For $x\in\mathbb{R}$, define $\mathrm{clip}(x,\ell,u)=\min\{u,\max\{\ell,x\}\}$ and
\[
\mathrm{clip}_{\mathbb{Z}}(x,K_{\min},K_{\max})
=\min\{K_{\max},\max\{K_{\min},\lfloor x+1/2\rfloor\}\}.
\]
Given $s_t$, the selector levers are deterministic clipped affine maps:
\[
\begin{aligned}
\alpha_t
&=\mathrm{clip}\!\left(\alpha_0+c_1(u_t^{(\mathrm{unc})}-0.5)-c_2(\widehat{e}_t-0.5),
  \underline{\alpha},\overline{\alpha}\right),\\
\lambda_t
&=\mathrm{clip}\!\left(\lambda_0+d_1(\widehat{e}_t-0.5),
  \underline{\lambda},\overline{\lambda}\right),\\
K_t
&=\mathrm{clip}_{\mathbb{Z}}\!\left(K_{\mathrm{base}}(0.8+0.4\widehat{e}_t),
  K_{\min},K_{\max}\right),\\
\Delta_t
&=\mathrm{clip}\!\left(\Delta_{\mathrm{base}}(0.8+0.6\widehat{e}_t),
  \underline{\Delta},\overline{\Delta}\right).
\end{aligned}
\]
The coefficients $\alpha_0,c_1,c_2,\lambda_0,d_1,K_{\mathrm{base}},\Delta_{\mathrm{base}}$ are nonnegative pre-specified constants and are not adapted during evaluation. The experiments use $\alpha_t\in[0.05,0.90]$, $\lambda_t\in[0.05,0.60]$, $K_t\in\{2,\ldots,8\}$, and $\Delta_t\in[0.06,0.22]$. The target value used by the shaping bonus is
\[
  m_t=\mathrm{clip}(\widehat{k}_t+\Delta_t,0,1).
\]

\paragraph{Replay inputs.}
For each round, the replay contract is the trace defined above. Candidate identities, target-value proxies $d(i)$, uncertainty widths $\sigma_t(i)$, pairwise similarities, novelty memory $H_t$, and deterministic tie-breaking order are either stored directly in the trace or reconstructed from immutable item metadata using logged identifiers. If any such input is computed from protected data, it must be a previous DP output or be separately DP-accounted for an end-to-end claim. Merely fixing it as a conditioned observable protects the learner output only conditional on that value; it does not protect the input itself.

\section{Window-enforcing selector (SCPO)}
\label{app:scpo}

We describe the \emph{constrained policy-optimised} (SCPO) variant of $\mathcal{S}$ referenced in Lemma~\ref{lem:drift}. SCPO replaces the default greedy MMR assembly step with a feasibility-constrained selector: at round $t$ it solves
\[
\begin{aligned}
S_t \in \arg\max_{S \subset C_t,\, |S| = K_t}\quad
& (1 - \lambda_t) \sum_{i \in S} \tilde\mu_t(i)
  - \lambda_t \max_{\substack{i,j \in S\\ i\neq j}}\mathrm{sim}(i,j)\\
\text{s.t.}\quad
& |d(i)-m_t|\le w\quad \forall i\in S,
\end{aligned}
\]
with a bounded relaxation: if no size-$K_t$ subset of $C_t$ satisfies the window, the window is relaxed to the smallest $w'_t \geq w$ under which at least $K_t$ candidates are feasible. Under SCPO the fixed-window event $\mathsf{Win}_t(w)$ in Lemma~\ref{lem:drift} holds by construction whenever the relaxation does not trigger. On relaxation rounds, Lemma~\ref{lem:drift} applies with the realized width $w'_t$, so the exact two-round envelope is $\eta+2\Delta_{\max}+w'_t+w'_{t+1}$ rather than the simpler fixed-window envelope.

\section{Supplementary proof details}
\label{app:proof-details}

This section records the formal statements and proof details shortened from the main text. They are included here so the main body can focus on intuition without dropping the formal argument.

\paragraph{Formal statements.}

\begin{lemma}[Post-processing privacy for decoupled selection]
\label{lem:postproc}
Let $\mathcal{L}:D\mapsto(\theta_1,\ldots,\theta_T)$ be $(\varepsilon,\delta)$-DP at record level. Let $\mathcal{S}\in\Pi_{\mathrm{gr}}$ be deterministic and receive at each round only the learner-derived score $\mu_t=f(\theta_t,\cdot)$, equivalently $\theta_t$ through the fixed scoring map $f$, and side information that is public or independent of $D$ or is deterministic post-processing of previous DP outputs. Then $\pi_{\mathcal{L},\mathcal{S}}(D)=(S_1,\ldots,S_T)$ is $(\varepsilon,\delta)$-DP. If $\mathcal{S}$ also consumes auxiliary DP outputs with budgets $(\varepsilon_j,\delta_j)$, the pipeline guarantee follows by standard composition. If other inputs are fixed by conditioning, the statement is conditional on those inputs and does not protect them.
\end{lemma}

\begin{corollary}[Audit trace privacy]
\label{cor:audit}
Publishing the replay trace $(\tau_t)_{t=1}^T$ defined in the supplementary setting details spends no additional privacy budget whenever every non-learner component is public or independent, derived from a previous DP output, or separately DP-accounted, and every learner-derived component is a deterministic function of $\theta_t$. Conditioned raw inputs and traces exposing a non-private anchor trained on $D$ are not covered by an end-to-end claim.
\end{corollary}

\begin{lemma}[User-level lift]
\label{lem:userdp}
Under Lemma~\ref{lem:postproc}'s hypotheses, if each user contributes at most $B$ records and $\mathcal{L}$ is $(\varepsilon,\delta)$-DP at record level, then the pipeline is
\[
  \left(B\varepsilon,\;\frac{e^{B\varepsilon}-1}{e^\varepsilon-1}\delta\right)\text{-DP}
\]
at user level. In particular, the second component is at most $B e^{(B-1)\varepsilon}\delta$.
\end{lemma}

\begin{lemma}[Selector-side target-value drift]
\label{lem:drift}
Let $A_t$ be the accepted items in round $t$ and $\bar d(A_t)=|A_t|^{-1}\sum_{i\in A_t}d(i)$. Suppose the selector enforces $|d(i)-m_t|\le w_t$ for every accepted item and $|m_{t+1}-m_t|\le \eta+2\Delta_{\max}$. Then
\[
  |\bar d(A_{t+1})-\bar d(A_t)| \le \eta+2\Delta_{\max}+w_t+w_{t+1}.
\]
For a fixed window $w$, this becomes $\eta+2\Delta_{\max}+2w$.
\end{lemma}

\begin{theorem}[Margin-certified selector stability]
\label{thm:selector-stability}
Consider a deterministic selector that builds an ordered slate through $K$ greedy choices. At step $q$, with prefix $S_{q-1}$, write the selector objective as $F_q^\mu(i\mid S_{q-1})$. Let $\mu^0$ be a reference score vector and let $i_q^0$ be the item selected under $\mu^0$. Define the stepwise reference margin
\[
  \Gamma_q=\min_{j\in C\setminus S_{q-1}^0,\;j\ne i_q^0}
  \left[
    F_q^{\mu^0}(i_q^0\mid S_{q-1}^0)-F_q^{\mu^0}(j\mid S_{q-1}^0)
  \right],
\]
and let $\Gamma=\min_q\Gamma_q$. If a perturbation envelope $M$ satisfies
\[
  |F_q^\mu(i\mid S_{q-1}^0)-F_q^{\mu^0}(i\mid S_{q-1}^0)|\le M
\]
for every step $q$ and remaining item $i$ along the reference trajectory, and if $\Gamma>0$ and $M<\Gamma/2$, then the ordered slate produced under $\mu$ is identical to the ordered slate produced under $\mu^0$.
\end{theorem}

\begin{corollary}[Raw top-$K$ score-ranking stability]
\label{cor:topk-stability}
Let $\mu^0:C\to\mathbb{R}$ be a fixed anchor score vector with anchor gap
\[
\gamma=
\min_{\substack{i\in\mathrm{TopK}(\mu^0)\\j\notin\mathrm{TopK}(\mu^0)}}
\bigl[\mu^0(i)-\mu^0(j)\bigr]>0.
\]
For the blended score $\mu^w(i)=w\mu^0(i)+(1-w)(\mu^0(i)+\xi_i)$ with $w\in[0,1)$, assume that for every originally top-$K$ item $i$ and outside item $j$,
\[
  \Pr[\xi_j-\xi_i\ge u]\le \exp(-u^2/(2\tau^2)).
\]
Then
\[
  \Pr[\mathrm{TopK}(\mu^w)\ne \mathrm{TopK}(\mu^0)]
  \le
  K(|C|-K)\exp\!\left(-\frac{\gamma^2}{2(1-w)^2\tau^2}\right).
\]
If the item errors are independent zero-mean sub-Gaussian with coordinate proxy $\sigma^2$, then $\tau^2=2\sigma^2$ and the exponent becomes $-\gamma^2/[4(1-w)^2\sigma^2]$. At $w=1$, the adaptive score is ignored and the flip probability is zero.
\end{corollary}

\begin{corollary}[Biased anchor]
\label{cor:approx-anchor}
If the anchor satisfies $\|\mu^0-\mu^\star\|_\infty\le\beta$, the same bound holds with the effective gap $\gamma_\beta=(\gamma-2\beta)_+$ in place of $\gamma$.
\end{corollary}

\begin{proposition}[Greedy-selector stability under stepwise margin]
\label{prop:selector-margin}
\label{cor:selector-margin}
For a fixed round, let $\Gamma_t$ be the minimum stepwise margin of the greedy selector under $\mu^0$, and define
\[
M_t(\mu,\mu^0)
=(1-\lambda_t)\bigl(\|\mu-\mu^0\|_\infty+\|b_t^\mu-b_t^{\mu^0}\|_\infty\bigr)
+\|n_t^\mu-n_t^{\mu^0}\|_\infty .
\]
If $\Gamma_t>0$ and $M_t(\mu,\mu^0)<\Gamma_t/2$, then the greedy ordered slate under $\mu$ is identical to the greedy ordered slate under $\mu^0$.
\end{proposition}

\begin{corollary}[Sub-Gaussian greedy-slate stability]
\label{cor:greedy-stability}
In the fixed-bonus case $b_t^\mu=b_t^{\mu^0}$ and $n_t^\mu=n_t^{\mu^0}$, if each coordinate error is two-sided sub-Gaussian with proxy $\sigma^2$ and $\lambda_t<1$, then
\[
  \Pr[S_t(\mu)\ne S_t(\mu^0)]
  \le
  2|C_t|\exp\!\left(
    -\frac{\Gamma_t^2}{8(1-\lambda_t)^2\sigma^2}
  \right).
\]
\end{corollary}

\begin{proposition}[LinUCB recovery in the exploration-only regime]
\label{prop:regret}
If target shaping, diversity, novelty, and variable slate size are disabled, $K_t=1$, and the exploration bonus is the standard LinUCB posterior width, the selector recovers the usual optimistic linear-bandit rule and obtains $\tilde O(d\sqrt{T})$ regret under the standard realizability and confidence assumptions~\cite{Li2010,Chu2011,Abbasi2011}.
\end{proposition}

\paragraph{Post-processing and audit trace.}
For Lemma~\ref{lem:postproc}, fix any side-information sequence $z=(z_1,\ldots,z_T)$ that is public or independent of $D$, or itself covered by the applicable DP composition. The full slate sequence is then a deterministic map
\[
  \pi_{\mathcal{L},\mathcal{S}}(D)=h_z(\mathcal{L}(D)).
\]
For neighboring datasets $D\sim D'$ and any measurable event $E$ over slate sequences,
\[
\Pr[h_z(\mathcal{L}(D))\in E]
\le e^\varepsilon \Pr[h_z(\mathcal{L}(D'))\in E]+\delta,
\]
because $h_z^{-1}(E)$ is an event over learner outputs. The audit trace corollary applies the same argument to the extended deterministic output $\tau_t$ defined in this appendix. If raw state or candidates are simply fixed, the displayed inequality is conditional on them; if the trace exposes a non-private anchor score trained on $D$, an end-to-end claim no longer follows.

\paragraph{User-level lift.}
For Lemma~\ref{lem:userdp}, let $D\sim_uD'$ differ in all records of one user and suppose the user contributes at most $B$ records. There exists a chain $D=D_0,D_1,\ldots,D_B=D'$ in which adjacent datasets are record-neighboring. Applying record-level DP along the chain gives the standard group-privacy bound
\[
  \Pr[\mathcal{L}(D)\in E]
  \le
  e^{B\varepsilon}\Pr[\mathcal{L}(D')\in E]
  +
  \left(\sum_{k=0}^{B-1}e^{k\varepsilon}\right)\delta.
\]
Since $\sum_{k=0}^{B-1}e^{k\varepsilon}=(e^{B\varepsilon}-1)/(e^\varepsilon-1)\le Be^{(B-1)\varepsilon}$, the displayed user-level statement follows, and Lemma~\ref{lem:postproc} transfers it to the pipeline.

\paragraph{Target-value drift.}
Under window enforcement, every accepted item in round $t$ lies in $[m_t-w_t,m_t+w_t]$, so $\bar d(A_t)$ also lies in that interval. Similarly, $\bar d(A_{t+1})\in[m_{t+1}-w_{t+1},m_{t+1}+w_{t+1}]$. Therefore
\[
|\bar d(A_{t+1})-\bar d(A_t)|
\le
|m_{t+1}-m_t|+w_t+w_{t+1}
\le
\eta+2\Delta_{\max}+w_t+w_{t+1}.
\]
The soft-window MMR selector only inherits this statement on the event that both realized accepted sets satisfy the window; the exact deterministic statement belongs to SCPO or any selector that enforces the window by construction.

\paragraph{Margin-certified selector stability.}
For Theorem~\ref{thm:selector-stability}, fix a step $q$ and assume by induction that the prefixes under $\mu$ and $\mu^0$ agree through $q-1$. For any competing item $j$ at step $q$,
\[
\begin{aligned}
F_q^\mu(i_q^0\mid S_{q-1}^0)-F_q^\mu(j\mid S_{q-1}^0)
&\ge
F_q^{\mu^0}(i_q^0\mid S_{q-1}^0)-F_q^{\mu^0}(j\mid S_{q-1}^0)-2M\\
&\ge \Gamma-2M>0.
\end{aligned}
\]
Thus the same item remains the unique maximizer. Induction over $q=1,\ldots,K$ preserves the entire ordered slate.

\paragraph{Top-$K$ stability.}
For Corollary~\ref{cor:topk-stability}, a top-$K$ set can change only if an anchor top-$K$ item $i$ is overtaken by a non-top-$K$ item $j$. Since $\mu^0(i)-\mu^0(j)\ge\gamma$,
\[
\mu^w(j)>\mu^w(i)
\quad\Rightarrow\quad
\xi_j-\xi_i>\gamma/(1-w).
\]
The pairwise tail condition gives probability at most $\exp[-\gamma^2/(2(1-w)^2\tau^2)]$ for each cross-boundary pair, and a union bound over $K(|C|-K)$ pairs gives the result. If $\|\mu^0-\mu^\star\|_\infty\le\beta$, the available gap decreases by at most $2\beta$, yielding $\gamma_\beta=(\gamma-2\beta)_+$.

\paragraph{Greedy slate stability.}
For Proposition~\ref{prop:selector-margin}, the shaped greedy objective perturbation at any candidate is bounded by
\[
M_t(\mu,\mu^0)
=(1-\lambda_t)(\|\mu-\mu^0\|_\infty+\|b_t^\mu-b_t^{\mu^0}\|_\infty)
+\|n_t^\mu-n_t^{\mu^0}\|_\infty.
\]
Theorem~\ref{thm:selector-stability} then applies with $M=M_t$ and $\Gamma=\Gamma_t$. In the fixed-bonus case, $M_t=(1-\lambda_t)\|\mu-\mu^0\|_\infty$. If all coordinate errors are two-sided sub-Gaussian with proxy $\sigma^2$, then
\[
\Pr[(1-\lambda_t)\|\mu-\mu^0\|_\infty\ge \Gamma_t/2]
\le
2|C_t|\exp[-\Gamma_t^2/(8(1-\lambda_t)^2\sigma^2)],
\]
which gives Corollary~\ref{cor:greedy-stability}.

\paragraph{LinUCB recovery.}
For Proposition~\ref{prop:regret}, disabling target shaping, diversity, novelty, and slate-size adaptation leaves a single-item selector with objective
\[
  x_i^\top\hat\theta_t+\alpha_t\|x_i\|_{A_t^{-1}}.
\]
Choosing $\alpha_t$ as the standard LinUCB confidence radius recovers the optimistic linear-bandit rule; the $\tilde O(d\sqrt{T})$ regret rate follows from the standard elliptical-potential analysis under the usual bounded-feature, bounded-noise, and realizability assumptions.

\section{Supplementary methodology details}
\label{app:methodology-details}

\begin{table}[H]
\centering
\small
\setlength{\tabcolsep}{6pt}
\caption{Corpus summary. Density (\textperthousand) is interactions per thousand possible user--item pairs. Splits are user-stratified at random seed 42.}
\label{tab:datasets-app}
\begin{tabular}{lcc}
\toprule
& \textbf{OULAD} & \textbf{EdNet} \\
\midrule
\# Users                    & 23{,}351                   & 296{,}701 \\
\# Items                    & 188                        & 11{,}555 \\
\# Interactions             & 173{,}739                  & 23{,}384{,}480 \\
Density (\textperthousand)  & 39.58                      & 6.82 \\
Avg. interactions per user  & 7.44                       & 78.8 \\
\bottomrule
\end{tabular}
\end{table}

\begin{figure}[H]
  \centering
  \includegraphics[width=\textwidth]{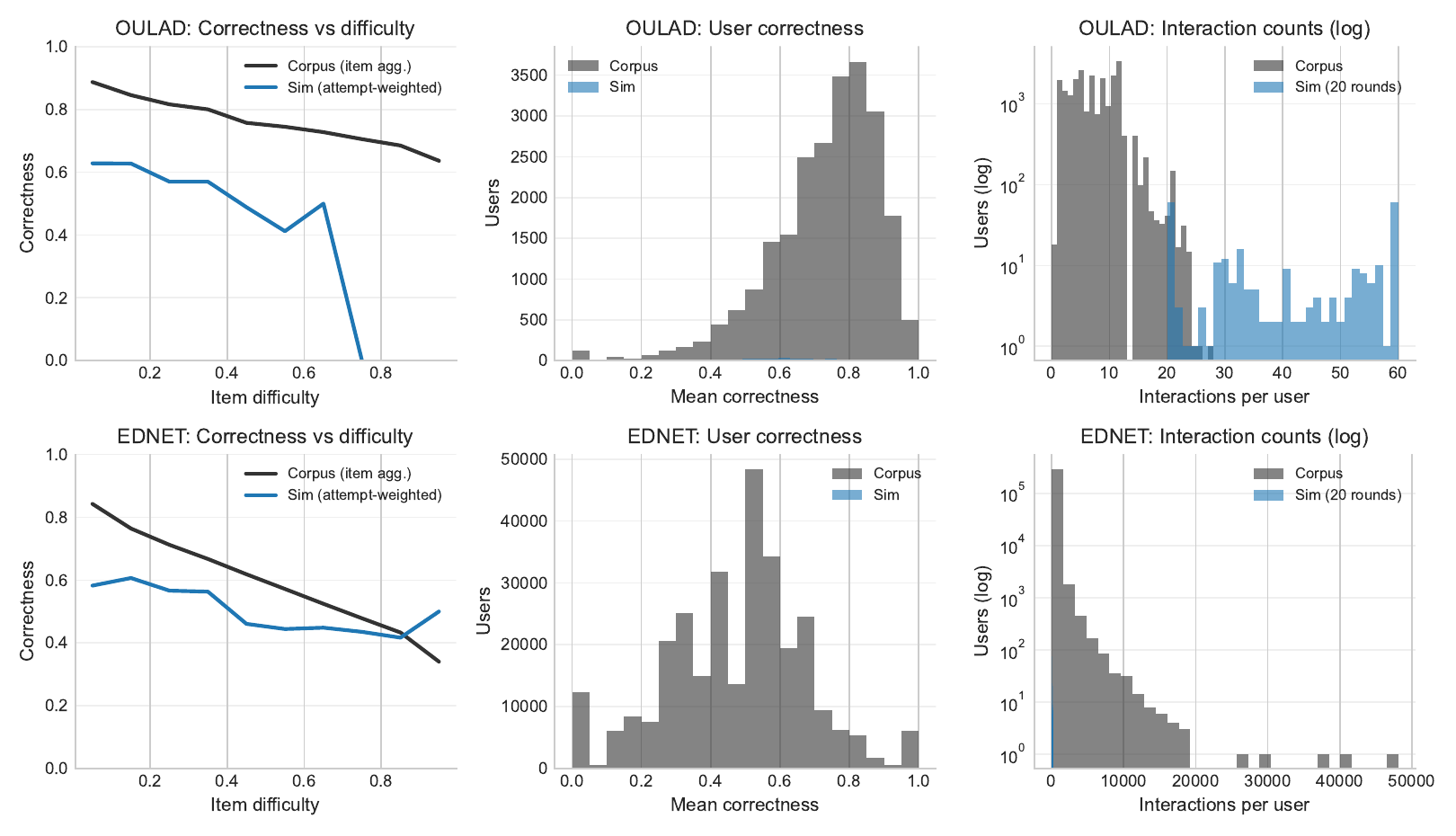}
  \Description{Simulator validation comparing corpus aggregates to closed-loop simulation outcomes on OULAD and EdNet.}
  \caption{Simulator validation: target-value vs.\ correctness, user-level correctness distributions, and interaction counts (log scale).}
  \label{fig:sim-validation-app}
\end{figure}

\begin{table}[H]
\centering
\small
\setlength{\tabcolsep}{4pt}
\caption{Nominal privacy configurations recorded by the experimental harness. The accountant labels use the stated $(q,\sigma,\delta,T)$ values, but the released artifact does not establish the sampler, population denominator, or protected-unit gradient aggregation needed to verify them as DP guarantees.}
\label{tab:dp-hyperparams-app}
\begin{tabular}{lccc}
\toprule
\textbf{Setting} & \textbf{Standard} & \textbf{Strong} & \textbf{Locked} \\
\midrule
Approx.\ $\varepsilon$ ($T = 20$) & $\approx 0.36$ & $\approx 0.16$ & $\approx 0.07$ \\
Approx.\ $\varepsilon$ ($T = 50$) & $\approx 0.58$ & $\approx 0.26$ & $\approx 0.11$ \\
Clipping norm $C$                 & 1.0             & 1.0             & 1.0 \\
Noise multiplier $\sigma$         & 1.2             & 2.0             & 3.0 \\
Sampling rate $q$                 & 0.020           & 0.015           & 0.010 \\
$\delta$                          & $10^{-5}$      & $10^{-5}$      & $10^{-5}$ \\
\bottomrule
\end{tabular}
\end{table}

\paragraph{Closed-loop protocol.}
Evaluation proceeds in closed loop: each round produces a slate, the response simulator emits item-level outcomes, and the policy state updates from observed responses. Candidate pools are generated once per round and replayed across policies under matched seeds. The design is factorial over four response simulators (\textsc{irt}, \textsc{mirt}, \textsc{zpd}, \textsc{contextual\_zpd}) and four user profiles (\textsc{struggling}, \textsc{steady}, \textsc{advancing}, \textsc{high\_flyer}), yielding paired session-level comparisons. Component-removal tests turn off target shaping, exploration, diversity, or slate-size adaptation one at a time to attribute changes to selector terms.

\paragraph{Raw stability experiments.}
The fixed-margin experiments use $|C|=100$, $K=10$, controlled gaps $\gamma\in\{0.35,0.50,0.70,0.90\}$, 64 independently permuted pools per gap, and 500 noise draws per pool, for 32{,}000 trials per cell. Anchor weights are $0$, $0.50$, and $0.75$. The descriptive fit uses the eight cells with smoothed Flip@$K$ in $[10^{-4},0.25]$; this rule was not preregistered. Its slope is $-0.220$ (SE 0.0042; 95\% $t$ CI $[-0.231,-0.210]$). For all anchor diagnostics, each source score vector is standardized within the candidate pool before blending. The real-anchor OULAD run uses 64 user-specific pools and 200 noise draws per pool; MovieLens-25M and Amazon Musical Instruments use 48 user-specific pools and the same number of draws. Low-rank correlated noise is summarized by the empirical pairwise scale $\hat\tau^2$ used in Corollary~\ref{cor:topk-stability}.

\section{Supplementary related-work object map}
\label{app:related-map}

\begin{table}[H]
\centering
\small
\caption{Analytic object emphasized by related families. The entries describe what each family usually studies, not whether the family is deficient.}
\label{tab:benefits-app}
\begin{tabularx}{\textwidth}{p{0.19\textwidth}p{0.25\textwidth}Xp{0.26\textwidth}}
\toprule
Family & Primary object usually analyzed & Usually outside the primary analysis & Relation to this paper \\
\midrule
DP recommendation \cite{McSherryMironov2009,Abadi2016,Dwork2014} & Privacy of the learned model or training algorithm & Deterministic post-layer, audit trace, and score-to-slate stability & Private regime lifts learner DP through a deterministic selector only under admissible/accounted side inputs \\
Re-ranking / slate construction \cite{CarbonellGoldstein1998,KunaverPozrl2017,Steck2018,SinghJoachims2018,Geyik2019} & Deterministic slate objective for diversity, calibration, fairness, or novelty & Learner privacy and explicit scorer-noise model & This paper treats the selector as a bounded surface and studies how it composes with learner privacy and score stability \\
Contextual \& private bandits \cite{Li2010,Chu2011,Abbasi2011,Shariff2018,Mishra2015} & Exploration/regret, sometimes with privacy built into the bandit & Replayable deterministic selection layer and shaped/diverse slate assembly & Proposition~\ref{prop:regret} recovers LinUCB only in the reduced exploration-only regime \\
Safe / constrained RL for recs \cite{Ie2019SlateQ,Achiam2017CPO} & Constraint satisfaction inside a learned policy & Privacy composition and deterministic audit replay & Lemma~\ref{lem:drift} gives a selector-side drift bound under window enforcement without retraining the learner \\
Off-policy learning / evaluation \cite{Swaminathan2015CRM} & Estimating or optimizing policy value from logged data & Structural privacy/audit/stability guarantees of the deployed post-layer & Complementary: such estimators could train $\mathcal{L}$ before the deterministic selector is applied \\
Governance / responsible AI \cite{SladePrinsloo2013,drachsler2016delicate,Mitchell2019ModelCards} & Documentation, review, and process artifacts & Formal algorithmic replay and privacy accounting & The audit trace is a formal decision artifact that governance processes could inspect \\
\midrule
\textbf{Private regime (this paper)} & DP learner plus deterministic selector with admissible/accounted side inputs & Non-private auxiliary anchor trained on $D$ & Privacy, audit replay, and drift under stated assumptions; score stability only if the anchor is admissible or DP-accounted \\
\textbf{Anchored regime (this paper)} & Fixed/non-private anchor plus deterministic selector & End-to-end pipeline DP & Top-$K$ score stability and greedy-slate stability under a margin condition; privacy claim is deliberately excluded \\
\bottomrule
\end{tabularx}
\end{table}

\section{Additional tables}
\label{app:additional-tables}

Full versions of additional result tables are reproduced below in compact form for reference; they were generated by the same pipeline as the main-text summaries under identical seeds and candidate pools.

\begin{table}[H]
  \centering
  \scriptsize
  \caption{Paired deltas vs.\ \textsc{ALS} (mean [95\% CI]) for utility and engagement metrics and absolute target-value jump. Probe deltas are reported on the $\times 10^{-3}$ scale.}
  \label{tab:rq1-summary-app}
  \begin{adjustbox}{max width=\textwidth}
    \begin{tabular}{lllrrrr}
\toprule
Dataset & Policy & $\Delta$ Corr.\ $\uparrow$ & $\Delta$ Post-test ($\times 10^{-3}$) $\uparrow$ & $\Delta$ Telemetry ($\times 10^{-3}$) $\uparrow$ & $\Delta$ Attempts/Round $\uparrow$ & Abs. diff. jump $\downarrow$ \\
\midrule
OULAD & Adaptive $-$ ALS & \underline{0.121 [0.105, 0.137]} & \textbf{14.824 [13.172, 16.350]} & \textbf{69.248 [44.959, 91.789]} & \underline{0.714 [0.650, 0.777]} & \textbf{0.105} \\
OULAD & Adaptive-SCPO $-$ ALS & \textbf{0.169 [-0.051, 0.390]} & \underline{13.485 [-7.606, 34.577]} & \underline{64.098 [26.179, 102.016]} & \textbf{0.975 [0.950, 1.000]} & \underline{0.145} \\
EdNet & Adaptive $-$ ALS & \underline{0.026 [0.014, 0.039]} & 0.032 [-1.062, 1.148] & \underline{22.168 [1.875, 41.991]} & \underline{0.140 [0.115, 0.168]} & \underline{0.195} \\
EdNet & Adaptive-SCPO $-$ ALS & 0.023 [0.011, 0.034] & \underline{1.530 [0.505, 2.597]} & 10.883 [-7.953, 31.653] & 0.134 [0.109, 0.160] & \textbf{0.173} \\
EdNet & Adaptive+ALS ($w=0.75$) $-$ ALS & \textbf{0.037 [0.025, 0.049]} & \textbf{3.137 [2.016, 4.338]} & \textbf{25.597 [4.873, 45.559]} & \textbf{0.241 [0.211, 0.271]} & -- \\
\bottomrule
\end{tabular}

  \end{adjustbox}
\end{table}

\begin{table}[H]
  \centering
  \scriptsize
  \caption{Response-model robustness (mean $\pm$ SD, min) versus \textsc{ALS}. Probe deltas on the $\times 10^{-3}$ scale.}
  \label{tab:choice-robustness-app}
  \begin{adjustbox}{max width=\textwidth}
    \begin{tabular}{lllll}
\toprule
Dataset & Policy & $\Delta$ Corr.\ $\uparrow$ & $\Delta$ Post-test ($\times 10^{-3}$) $\uparrow$ & $\Delta$ Uptake $\uparrow$ \\
\midrule
OULAD & Adaptive $-$ ALS & \textbf{0.114±0.006 (min 0.106)} & \textbf{12.456±1.840 (min 10.034)} & \textbf{0.144±0.009 (min 0.130)} \\
OULAD & Adaptive+ALS ($w=0.75$) $-$ ALS & \underline{0.010±0.006 (min 0.002)} & \underline{1.873±0.492 (min 1.164)} & \underline{0.026±0.002 (min 0.024)} \\
EdNet & Adaptive $-$ ALS & \underline{0.025±0.010 (min 0.014)} & \underline{2.647±1.672 (min 0.477)} & \underline{0.039±0.015 (min 0.021)} \\
EdNet & Adaptive+ALS ($w=0.75$) $-$ ALS & \textbf{0.033±0.008 (min 0.028)} & \textbf{3.469±1.457 (min 2.300)} & \textbf{0.060±0.007 (min 0.054)} \\
\bottomrule
\end{tabular}

  \end{adjustbox}
\end{table}

\begin{table}[H]
\centering
\scriptsize
\setlength{\tabcolsep}{4pt}
\caption{The eight fixed-margin calibration cells used for the descriptive exponent fit. Here $x=\gamma^2/((1-w)^2\sigma^2)$ and $M=K(|C|-K)$. The non-saturation rule is smoothed Flip@$K\in[10^{-4},0.25]$; it was not preregistered. OLS gives slope $-0.220$ (SE 0.0042; 95\% $t$ CI $[-0.231,-0.210]$) against reference $-1/4$ ($R^2=0.998$).}
\label{tab:fixed-margin-fit-cells}
\begin{adjustbox}{max width=\textwidth}
\begin{tabular}{rrrrrrr}
\toprule
$\gamma$ & $w$ & $\sigma$ & Trials & $x$ & Smoothed Flip@$K$ & $\log(\hat p/M)$ \\
\midrule
0.35 & 0.50 & 0.12 & 32{,}000 & 34.0278 & 0.01230 & -11.2008 \\
0.35 & 0.75 & 0.25 & 32{,}000 & 31.3600 & 0.02417 & -10.5250 \\
0.50 & 0.50 & 0.18 & 32{,}000 & 30.8642 & 0.02630 & -10.4407 \\
0.50 & 0.75 & 0.35 & 32{,}000 & 32.6531 & 0.01708 & -10.8724 \\
0.70 & 0.00 & 0.12 & 32{,}000 & 34.0278 & 0.01223 & -11.2059 \\
0.70 & 0.50 & 0.25 & 32{,}000 & 31.3600 & 0.02333 & -10.5605 \\
0.90 & 0.00 & 0.18 & 32{,}000 & 25.0000 & 0.09229 & -9.1852 \\
0.90 & 0.50 & 0.35 & 32{,}000 & 26.4490 & 0.06517 & -9.5332 \\
\bottomrule
\end{tabular}

\end{adjustbox}
\end{table}

\begin{table}[H]
\centering
\scriptsize
\caption{Fixed-margin calibration fits by noise family. Independent Gaussian noise is fit against the closed-form specialization's $\sigma^2$ scale; low-rank correlated noise is fit against the pairwise $\hat\tau_{\max}^2$ scale from Corollary~\ref{cor:topk-stability}.}
\label{tab:fixed-margin-pairtau-fit}
\begin{adjustbox}{max width=\columnwidth}
\begin{tabular}{llrrrr}
\toprule
Noise family & Fit axis & Slope & Reference & $R^2$ & Cells \\
\midrule
independent Gaussian & $\gamma^2/((1-w)^2\sigma^2)$ & -0.220 & -0.250 & 0.998 & 8 \\
low-rank r=4, $\rho$=0.50 & $\gamma^2/((1-w)^2\hat\tau_{\max}^2)$ & -0.506 & -0.500 & 0.997 & 12 \\
\bottomrule
\end{tabular}

\end{adjustbox}
\end{table}

\begin{table}[H]
  \centering
  \scriptsize
  \caption{Rating-threshold sensitivity for real-anchor raw Flip@$K$ checks. The stabilizing effect of anchor weight $w=0.75$ persists across binarization thresholds.}
  \label{tab:topk-flip-threshold-sensitivity}
  \begin{adjustbox}{max width=\textwidth}
\begin{tabular}{lrrrrrr}
\toprule
Dataset & Threshold & $\sigma$ & Flip@$K$ $w=0$ & Flip@$K$ $w=0.75$ & $\Delta$ Flip@$K$ & Rel. drop \\
\midrule
Amazon Musical Instruments & 3.500 & 0.050 & 0.262 & 0.045 & 0.216 & 0.827 \\
Amazon Musical Instruments & 3.500 & 0.100 & 0.490 & 0.110 & 0.380 & 0.776 \\
Amazon Musical Instruments & 3.500 & 0.200 & 0.768 & 0.254 & 0.514 & 0.669 \\
Amazon Musical Instruments & 4.000 & 0.050 & 0.301 & 0.082 & 0.219 & 0.727 \\
Amazon Musical Instruments & 4.000 & 0.100 & 0.512 & 0.162 & 0.350 & 0.683 \\
Amazon Musical Instruments & 4.000 & 0.200 & 0.781 & 0.306 & 0.475 & 0.608 \\
Amazon Musical Instruments & 4.500 & 0.050 & 0.272 & 0.076 & 0.196 & 0.720 \\
Amazon Musical Instruments & 4.500 & 0.100 & 0.492 & 0.133 & 0.359 & 0.730 \\
Amazon Musical Instruments & 4.500 & 0.200 & 0.758 & 0.270 & 0.488 & 0.644 \\
MovieLens-25M & 3.500 & 0.050 & 0.316 & 0.077 & 0.239 & 0.756 \\
MovieLens-25M & 3.500 & 0.100 & 0.528 & 0.161 & 0.367 & 0.694 \\
MovieLens-25M & 3.500 & 0.200 & 0.788 & 0.315 & 0.474 & 0.601 \\
MovieLens-25M & 4.000 & 0.050 & 0.352 & 0.102 & 0.250 & 0.709 \\
MovieLens-25M & 4.000 & 0.100 & 0.605 & 0.176 & 0.429 & 0.709 \\
MovieLens-25M & 4.000 & 0.200 & 0.859 & 0.356 & 0.504 & 0.586 \\
MovieLens-25M & 4.500 & 0.050 & 0.254 & 0.081 & 0.173 & 0.682 \\
MovieLens-25M & 4.500 & 0.100 & 0.445 & 0.147 & 0.297 & 0.669 \\
MovieLens-25M & 4.500 & 0.200 & 0.697 & 0.253 & 0.444 & 0.636 \\
\bottomrule
\end{tabular}
  \end{adjustbox}
\end{table}

\begin{table}[H]
  \centering
  \scriptsize
  \caption{Real-anchor raw Flip@$K$ checks on one educational corpus and two ratings corpora. Higher anchor weight consistently reduces raw top-$K$ flips under both independent and low-rank correlated score noise.}
  \label{tab:topk-flip-real-anchor-summary}
  \begin{adjustbox}{max width=\textwidth}
\begin{tabular}{llrrrrrrr}
\toprule
Dataset & Noise & $\sigma$ & Flip@$K$ $w=0$ & Flip@$K$ $w=0.50$ & Flip@$K$ $w=0.75$ & $\Delta$ Flip@$K$ & Rel. drop & Jaccard $w=0.75$ \\
\midrule
Amazon Musical Instruments & independent & 0.050 & 0.266 & 0.126 & 0.066 & 0.200 & 0.751 & 0.988 \\
Amazon Musical Instruments & independent & 0.100 & 0.494 & 0.264 & 0.128 & 0.367 & 0.742 & 0.976 \\
Amazon Musical Instruments & independent & 0.200 & 0.752 & 0.493 & 0.262 & 0.491 & 0.652 & 0.951 \\
Amazon Musical Instruments & low-rank r=4, $\rho$=0.50 & 0.050 & 0.258 & 0.122 & 0.064 & 0.193 & 0.750 & 0.988 \\
Amazon Musical Instruments & low-rank r=4, $\rho$=0.50 & 0.100 & 0.478 & 0.253 & 0.123 & 0.355 & 0.742 & 0.977 \\
Amazon Musical Instruments & low-rank r=4, $\rho$=0.50 & 0.200 & 0.741 & 0.475 & 0.257 & 0.484 & 0.653 & 0.952 \\
MovieLens-25M & independent & 0.050 & 0.410 & 0.229 & 0.114 & 0.296 & 0.721 & 0.978 \\
MovieLens-25M & independent & 0.100 & 0.653 & 0.412 & 0.229 & 0.424 & 0.649 & 0.957 \\
MovieLens-25M & independent & 0.200 & 0.875 & 0.648 & 0.410 & 0.465 & 0.532 & 0.921 \\
MovieLens-25M & low-rank r=4, $\rho$=0.50 & 0.050 & 0.417 & 0.230 & 0.111 & 0.306 & 0.734 & 0.979 \\
MovieLens-25M & low-rank r=4, $\rho$=0.50 & 0.100 & 0.653 & 0.408 & 0.227 & 0.426 & 0.652 & 0.957 \\
MovieLens-25M & low-rank r=4, $\rho$=0.50 & 0.200 & 0.872 & 0.651 & 0.404 & 0.468 & 0.537 & 0.922 \\
OULAD & independent & 0.020 & 0.167 & 0.080 & 0.032 & 0.135 & 0.808 & 0.994 \\
OULAD & independent & 0.050 & 0.372 & 0.211 & 0.107 & 0.265 & 0.713 & 0.980 \\
OULAD & independent & 0.100 & 0.603 & 0.368 & 0.216 & 0.387 & 0.642 & 0.960 \\
OULAD & low-rank r=1, $\rho$=0.25 & 0.020 & 0.167 & 0.085 & 0.030 & 0.137 & 0.818 & 0.994 \\
OULAD & low-rank r=1, $\rho$=0.25 & 0.050 & 0.365 & 0.205 & 0.103 & 0.262 & 0.718 & 0.981 \\
OULAD & low-rank r=1, $\rho$=0.25 & 0.100 & 0.599 & 0.362 & 0.201 & 0.398 & 0.665 & 0.962 \\
OULAD & low-rank r=1, $\rho$=0.75 & 0.020 & 0.156 & 0.080 & 0.033 & 0.123 & 0.789 & 0.994 \\
OULAD & low-rank r=1, $\rho$=0.75 & 0.050 & 0.333 & 0.200 & 0.100 & 0.233 & 0.699 & 0.982 \\
OULAD & low-rank r=1, $\rho$=0.75 & 0.100 & 0.544 & 0.336 & 0.194 & 0.351 & 0.645 & 0.963 \\
OULAD & low-rank r=4, $\rho$=0.25 & 0.020 & 0.170 & 0.083 & 0.032 & 0.138 & 0.813 & 0.994 \\
OULAD & low-rank r=4, $\rho$=0.25 & 0.050 & 0.374 & 0.211 & 0.105 & 0.269 & 0.719 & 0.981 \\
OULAD & low-rank r=4, $\rho$=0.25 & 0.100 & 0.594 & 0.368 & 0.210 & 0.383 & 0.646 & 0.961 \\
OULAD & low-rank r=4, $\rho$=0.75 & 0.020 & 0.160 & 0.078 & 0.035 & 0.124 & 0.778 & 0.994 \\
OULAD & low-rank r=4, $\rho$=0.75 & 0.050 & 0.358 & 0.197 & 0.098 & 0.260 & 0.727 & 0.982 \\
OULAD & low-rank r=4, $\rho$=0.75 & 0.100 & 0.583 & 0.351 & 0.201 & 0.383 & 0.656 & 0.962 \\
\bottomrule
\end{tabular}
  \end{adjustbox}
\end{table}

\begin{table}[H]
\centering
\scriptsize
\caption{OULAD positive-margin gate for the selector-margin certificate. The raw anchor gap $\gamma$ and greedy selector margin $\Gamma_t$ are positive in all logged records, but $\Gamma_t$ is small, so certificate usefulness depends on small perturbation envelopes.}
\label{tab:oulad-margin-gate}
\begin{tabular}{lrrrr}
\toprule
Quantity & Mean & Median & Positive rate & Near-zero rate \\
\midrule
$\gamma$ & 0.08462 & 0.05775 & 1.000 & 0.000 \\
$\Gamma_t$ & 0.0008477 & 0.0007332 & 1.000 & 0.000 \\
\bottomrule
\end{tabular}
\end{table}

\begin{table}[H]
\centering
\scriptsize
\caption{OULAD selector-margin perturbation check. Certified means $M_t<\Gamma_t/2$ under the fixed-bonus envelope. Certified violations are zero across all 8,000 perturbation trials.}
\label{tab:oulad-selector-perturb}
\begin{tabular}{rrrrrr}
\toprule
$\sigma_{\mathrm{pert}}$ & Certified & Same order & Same set & Certified violations & Median $M_t/\Gamma_t$ \\
\midrule
5e-05 & 0.9962 & 1 & 1 & 0 & 0.1343 \\
0.0001 & 0.9844 & 1 & 1 & 0 & 0.2698 \\
0.0002 & 0.4306 & 0.9975 & 0.9988 & 0 & 0.537 \\
0.0005 & 0.01188 & 0.9125 & 0.9975 & 0 & 1.347 \\
0.001 & 0 & 0.6569 & 0.9875 & 0 & 2.725 \\
\bottomrule
\end{tabular}

\end{table}

\begin{table}[H]
\centering
\scriptsize
\caption{EdNet positive-margin gate for the selector-margin certificate. The raw anchor gap $\gamma$ and greedy selector margin $\Gamma_t$ are positive in all logged records.}
\label{tab:ednet-margin-gate}
\begin{tabular}{lrrrr}
\toprule
Quantity & Mean & Median & Positive rate & Near-zero rate \\
\midrule
$\gamma$ & 0.01007 & 0.00706 & 1.000 & 0.000 \\
$\Gamma_t$ & 0.0003949 & 0.0002794 & 1.000 & 0.000 \\
\bottomrule
\end{tabular}
\end{table}

\begin{table}[H]
\centering
\scriptsize
\caption{EdNet selector-margin perturbation check. Certified means $M_t<\Gamma_t/2$ under the fixed-bonus envelope. Certified violations are zero across all perturbation trials.}
\label{tab:ednet-selector-perturb}
\begin{tabular}{rrrrrr}
\toprule
$\sigma_{\mathrm{pert}}$ & Certified & Same order & Same set & Certified violations & Median $M_t/\Gamma_t$ \\
\midrule
5e-05 & 0.5931 & 0.9519 & 0.9775 & 0 & 0.3442 \\
0.0001 & 0.3944 & 0.8925 & 0.9587 & 0 & 0.682 \\
0.0002 & 0.1856 & 0.8275 & 0.9375 & 0 & 1.387 \\
0.0005 & 0.003125 & 0.6669 & 0.8962 & 0 & 3.452 \\
0.001 & 0 & 0.4275 & 0.8225 & 0 & 7.198 \\
\bottomrule
\end{tabular}
\end{table}

\begin{figure}[H]
  \centering
  \begin{subfigure}[t]{0.48\textwidth}
    \centering
    \includegraphics[width=\linewidth]{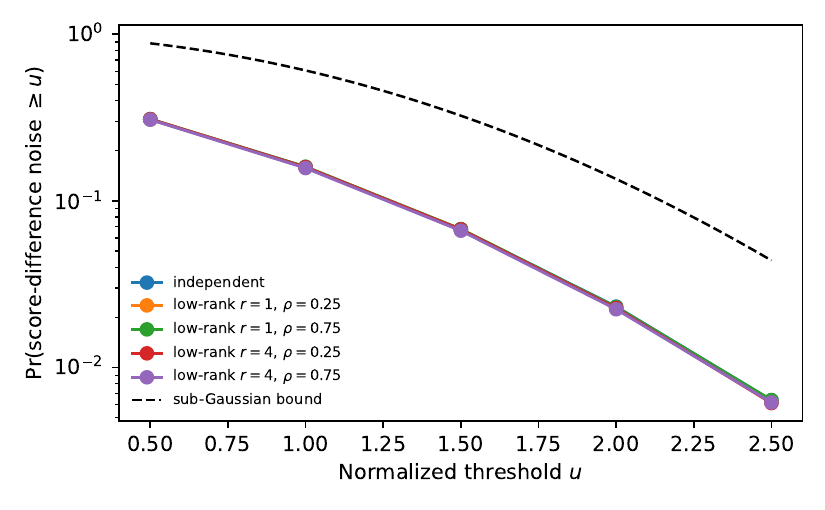}
    \caption{Pairwise score-noise tails}
  \end{subfigure}\hfill
  \begin{subfigure}[t]{0.48\textwidth}
    \centering
    \includegraphics[width=\linewidth]{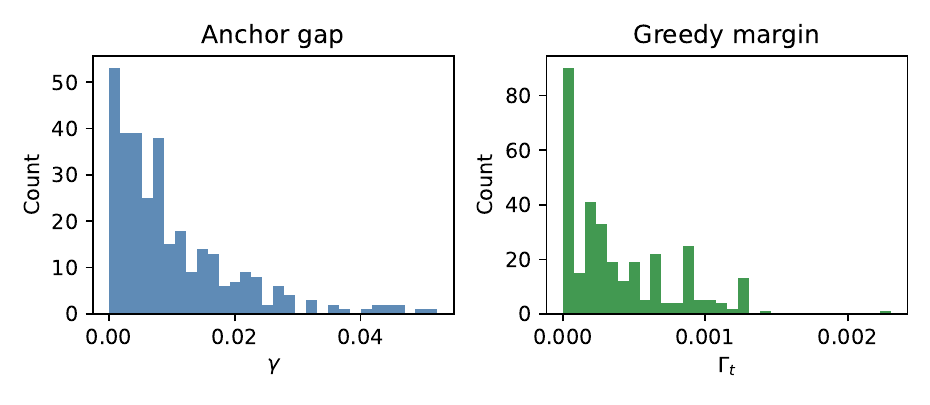}
    \caption{EdNet anchor and greedy margins}
  \end{subfigure}
  \caption{Full-size diagnostics behind the stability checks. Left: injected perturbations, including low-rank correlated noise, remain below the pairwise sub-Gaussian envelope when normalized by $\hat\tau$. Right: EdNet has positive anchor gaps and greedy margins in all logged candidate pools.}
  \Description{The left panel plots empirical upper-tail probabilities for independent and low-rank score noise against a dashed sub-Gaussian envelope; every measured curve remains below the envelope. The right panel shows EdNet histograms of positive anchor top-10 gaps and positive minimum greedy-step margins, with the greedy margins concentrated closer to zero.}
  \label{fig:tail-margin-diagnostics-app}
\end{figure}

\begin{table}[H]
\centering
\scriptsize
\caption{OULAD pairwise score-difference tail diagnostic. Noise differences are normalized by their pairwise $\tau$ proxy; all measured upper tails are below the sub-Gaussian envelope used by Corollary~\ref{cor:topk-stability}.}
\label{tab:pairwise-tail-oulad}
\begin{adjustbox}{max width=\textwidth}
\begin{tabular}{lrrrrrr}
\toprule
Noise & $r$ & $\rho$ & $\sigma$ & $u$ & Tail & Bound \\
\midrule
indep. & 0 & 0 & 0.02 & 0.5 & 0.3089 & 0.8825 \\
indep. & 0 & 0 & 0.02 & 1 & 0.1593 & 0.6065 \\
indep. & 0 & 0 & 0.02 & 1.5 & 0.06751 & 0.3247 \\
indep. & 0 & 0 & 0.02 & 2 & 0.02322 & 0.1353 \\
indep. & 0 & 0 & 0.02 & 2.5 & 0.006372 & 0.04394 \\
indep. & 0 & 0 & 0.05 & 1 & 0.1586 & 0.6065 \\
indep. & 0 & 0 & 0.05 & 2 & 0.02294 & 0.1353 \\
indep. & 0 & 0 & 0.1 & 1 & 0.1582 & 0.6065 \\
indep. & 0 & 0 & 0.1 & 2 & 0.02272 & 0.1353 \\
low-rank & 1 & 0.25 & 0.02 & 0.5 & 0.3092 & 0.8825 \\
low-rank & 1 & 0.25 & 0.02 & 1 & 0.1593 & 0.6065 \\
low-rank & 1 & 0.25 & 0.02 & 1.5 & 0.06719 & 0.3247 \\
low-rank & 1 & 0.25 & 0.02 & 2 & 0.02291 & 0.1353 \\
low-rank & 1 & 0.25 & 0.02 & 2.5 & 0.006254 & 0.04394 \\
low-rank & 1 & 0.25 & 0.05 & 1 & 0.1579 & 0.6065 \\
low-rank & 1 & 0.25 & 0.05 & 2 & 0.02288 & 0.1353 \\
low-rank & 1 & 0.25 & 0.1 & 1 & 0.1583 & 0.6065 \\
low-rank & 1 & 0.25 & 0.1 & 2 & 0.02261 & 0.1353 \\
\bottomrule
\end{tabular}

\end{adjustbox}
\end{table}

\begin{table}[H]
\centering
\scriptsize
\caption{Descriptive repeated-seed learner-update diagnostics on OULAD across nominal configurations. Each row uses 200 seeds, eight updates per seed, 16 fixed pools, and 160{,}000 pair draws that share seeds, pools, and scores. The artifact does not establish a complete DP mechanism for these configurations.}
\label{tab:dp-sgd-tail-oulad-tiers}
\begin{adjustbox}{max width=\textwidth}
\begin{tabular}{lrrrrrrrrrr}
\toprule
Tier & $\varepsilon$ & Mean $\hat\tau$ & Tail $u=2$ & Bound $u=2$ & Tail $u=2.5$ & Bound $u=2.5$ & Flip@$K$ $w=0$ & Flip@$K$ $w=0.50$ & Flip@$K$ $w=0.75$ & Jaccard $w=0.75$ \\
\midrule
\texttt{eps\_1} & 0.2284 & 0.6672 & 0.02063 & 0.1353 & 0.003025 & 0.04394 & 1 & 0.9988 & 0.945 & 0.7229 \\
\texttt{eps\_05} & 0.1022 & 0.6683 & 0.02058 & 0.1353 & 0.002975 & 0.04394 & 1 & 0.9988 & 0.9466 & 0.7218 \\
\texttt{eps\_02} & 0.04857 & 0.6688 & 0.02059 & 0.1353 & 0.002938 & 0.04394 & 1 & 0.9984 & 0.9466 & 0.7215 \\
\bottomrule
\end{tabular}
\end{adjustbox}
\end{table}

\begin{table}[H]
\centering
\scriptsize
\caption{Closed-loop paired deltas versus \textsc{ALS} across OULAD and EdNet. Probe deltas are reported on the $\times 10^{-3}$ scale.}
\label{tab:rq1-vs-als-app}
\begin{adjustbox}{max width=\textwidth}
\begin{tabular}{lllllll}
\toprule
Dataset & Comparison & $\Delta$ Corr. & $\Delta$ Post-test ($\times 10^{-3}$) & $\Delta$ Telemetry ($\times 10^{-3}$) & $\Delta$ Uptake & $\Delta$ Attempts/Round \\
\midrule
OULAD & Adaptive $-$ ALS & 0.121 [0.105, 0.137] & 14.824 [13.172, 16.350] & 69.248 [44.959, 91.789] & 0.150 [0.138, 0.162] & 0.714 [0.650, 0.777] \\
OULAD & Adaptive-SCPO $-$ ALS & 0.169 [-0.051, 0.390] & 13.485 [-7.606, 34.577] & 64.098 [26.179, 102.016] & 0.195 [0.190, 0.200] & 0.975 [0.950, 1.000] \\
OULAD & Adaptive+ALS ($w=0.75$) $-$ ALS & 0.197 [0.004, 0.390] & 15.862 [3.468, 28.257] & 68.054 [40.751, 95.357] & 0.240 [0.200, 0.280] & 1.200 [1.000, 1.400] \\
OULAD & ALS+Selector $-$ ALS & -0.002 [-0.013, 0.009] & -0.163 [-0.794, 0.485] & -12.940 [-31.719, 6.588] & 0.020 [0.019, 0.021] & 0.002 [-0.000, 0.004] \\
EdNet & Adaptive $-$ ALS & 0.026 [0.014, 0.039] & 0.032 [-1.062, 1.148] & 22.168 [1.875, 41.991] & 0.039 [0.034, 0.043] & 0.140 [0.115, 0.168] \\
EdNet & Adaptive-SCPO $-$ ALS & 0.023 [0.011, 0.034] & 1.530 [0.505, 2.597] & 10.883 [-7.953, 31.653] & 0.037 [0.032, 0.042] & 0.134 [0.109, 0.160] \\
EdNet & Adaptive+ALS ($w=0.75$) $-$ ALS & 0.037 [0.025, 0.049] & 3.137 [2.016, 4.338] & 25.597 [4.873, 45.559] & 0.057 [0.052, 0.063] & 0.241 [0.211, 0.271] \\
EdNet & ALS+Selector $-$ ALS & 0.004 [-0.005, 0.014] & 0.150 [-0.567, 0.884] & 15.519 [-0.431, 30.961] & 0.012 [0.010, 0.013] & -0.001 [-0.005, 0.004] \\
\bottomrule
\end{tabular}

\end{adjustbox}
\end{table}

\begin{table}[H]
\centering
\scriptsize
\caption{Alignment between telemetry probes and held-out utility probes. Sign agreement reports the fraction of paired sessions where the telemetry and utility deltas have the same sign.}
\label{tab:probe-alignment-app}
\begin{adjustbox}{max width=\textwidth}
\begin{tabular}{llrrrr}
\toprule
Dataset & Comparison & N & Pearson $r$ & Spearman $\rho$ & Sign agree. \\
\midrule
OULAD & Adaptive $-$ ALS & 256 & 0.312 & 0.268 & 0.684 \\
OULAD & ALS+Selector $-$ ALS & 256 & 0.534 & 0.513 & 0.711 \\
EdNet & Adaptive $-$ ALS & 256 & 0.337 & 0.354 & 0.566 \\
EdNet & Adaptive-SCPO $-$ ALS & 256 & 0.343 & 0.334 & 0.625 \\
EdNet & Adaptive+ALS ($w=0.75$) $-$ ALS & 256 & 0.318 & 0.323 & 0.633 \\
EdNet & ALS+Selector $-$ ALS & 256 & 0.559 & 0.551 & 0.852 \\
\bottomrule
\end{tabular}

\end{adjustbox}
\end{table}

\begin{table}[H]
\centering
\scriptsize
\caption{Configured-update scorer deltas under matched candidate pools. The nominal settings are not independently verified DP mechanisms; values are configuration minus reference paired deltas.}
\label{tab:rq2-dp-deltas-app}
\begin{adjustbox}{max width=\textwidth}
\begin{tabular}{lllll}
\toprule
Dataset & Privacy Setting & $\Delta$ Corr.\ $\uparrow$ & $\Delta$ Post-test ($\times 10^{-3}$) $\uparrow$ & $\Delta$ Uptake $\uparrow$ \\
\midrule
OULAD & Standard ($\varepsilon \approx 0.36$) & -0.045 [-0.055, -0.037] & 0.741 [-0.407, 1.891] & \underline{-0.036 [-0.040, -0.031]} \\
OULAD & Strong ($\varepsilon \approx 0.16$) & \textbf{-0.042 [-0.051, -0.033]} & \underline{0.883 [-0.306, 2.090]} & \textbf{-0.030 [-0.034, -0.026]} \\
OULAD & Locked ($\varepsilon \approx 0.07$) & \underline{-0.044 [-0.053, -0.035]} & \textbf{1.321 [0.173, 2.434]} & \textbf{-0.030 [-0.034, -0.026]} \\
EdNet & Standard ($\varepsilon \approx 0.36$) & \underline{-0.032 [-0.040, -0.023]} & -1.530 [-2.440, -0.610] & \textbf{-0.032 [-0.038, -0.027]} \\
EdNet & Strong ($\varepsilon \approx 0.16$) & \textbf{-0.030 [-0.039, -0.021]} & \textbf{-1.099 [-2.012, -0.139]} & \textbf{-0.032 [-0.037, -0.028]} \\
EdNet & Locked ($\varepsilon \approx 0.07$) & \underline{-0.032 [-0.041, -0.023]} & \underline{-1.215 [-2.183, -0.291]} & \textbf{-0.032 [-0.038, -0.028]} \\
\bottomrule
\end{tabular}

\end{adjustbox}
\end{table}

\begin{table}[H]
\centering
\scriptsize
\caption{Descriptive repeated-seed OULAD learner-update diagnostics at two update depths. Each row uses 200 seeds and 160{,}000 normalized pair draws that share seeds, pools, and item scores and therefore are not independent inferential observations. The released artifact does not establish a complete DP mechanism for these configurations.}
\label{tab:dp-sgd-tail-oulad-steps}
\begin{adjustbox}{max width=\textwidth}
\begin{tabular}{llrrrrrr}
\toprule
Steps & Tier & Mean $\varepsilon$ & Mean $\hat\tau$ & Tail $u=2$ & Tail $u=2.5$ & Flip@$K$ $w=0.75$ & Jaccard $w=0.75$ \\
\midrule
8 & Standard & 0.228 & 0.667 & 0.0206 & 0.0030 & 0.945 & 0.723 \\
8 & Strong & 0.102 & 0.668 & 0.0206 & 0.0030 & 0.947 & 0.722 \\
8 & Locked & 0.049 & 0.669 & 0.0206 & 0.0029 & 0.947 & 0.722 \\
80 & Standard & 0.738 & 1.166 & 0.0228 & 0.0062 & 0.985 & 0.629 \\
80 & Strong & 0.326 & 1.172 & 0.0227 & 0.0062 & 0.986 & 0.627 \\
80 & Locked & 0.154 & 1.174 & 0.0227 & 0.0062 & 0.986 & 0.627 \\
\bottomrule
\end{tabular}

\end{adjustbox}
\end{table}

\begin{table}[H]
\centering
\scriptsize
\caption{OULAD configured-update follow-up with explicit anchor-weight sweep. Values are nominal-configuration minus reference deltas under matched candidate pools for a 128-learner diagnostic run. Fixed \textsc{ALS} is unchanged by construction because it is not updated; zero change is not a utility comparison.}
\label{tab:rq2-dp-resilience-oulad-lambda}
\begin{adjustbox}{max width=\textwidth}
\begin{tabular}{llllllr}
\toprule
Dataset & Policy & Privacy Setting & $\Delta$ Corr. & $\Delta$ Post-test ($\times 10^{-3}$) & $\Delta$ Uptake & $n$ \\
\midrule
OULAD & Adaptive & Standard & -0.027 [-0.040, -0.016] & -0.373 [-0.768, -0.004] & -0.024 [-0.029, -0.019] & 128 \\
OULAD & Adaptive+ALS ($w=0.50$) & Standard & -0.025 [-0.038, -0.012] & -0.945 [-1.426, -0.499] & -0.028 [-0.036, -0.021] & 128 \\
OULAD & Adaptive+ALS ($w=0.75$) & Standard & -0.005 [-0.018, 0.006] & -0.227 [-0.641, 0.122] & -0.007 [-0.011, -0.003] & 128 \\
OULAD & ALS & Standard & 0.000 [0.000, 0.000] & 0.000 [0.000, 0.000] & 0.000 [0.000, 0.000] & 128 \\
OULAD & Adaptive & Strong & -0.027 [-0.039, -0.016] & -0.475 [-0.848, -0.100] & -0.024 [-0.029, -0.020] & 128 \\
OULAD & Adaptive+ALS ($w=0.50$) & Strong & -0.023 [-0.035, -0.011] & -0.905 [-1.359, -0.478] & -0.026 [-0.034, -0.020] & 128 \\
OULAD & Adaptive+ALS ($w=0.75$) & Strong & -0.007 [-0.020, 0.006] & -0.259 [-0.662, 0.094] & -0.007 [-0.011, -0.003] & 128 \\
OULAD & ALS & Strong & 0.000 [0.000, 0.000] & 0.000 [0.000, 0.000] & 0.000 [0.000, 0.000] & 128 \\
OULAD & Adaptive & Locked & -0.024 [-0.035, -0.014] & -0.431 [-0.789, -0.080] & -0.023 [-0.028, -0.018] & 128 \\
OULAD & Adaptive+ALS ($w=0.50$) & Locked & -0.025 [-0.038, -0.012] & -0.895 [-1.304, -0.486] & -0.028 [-0.036, -0.020] & 128 \\
OULAD & Adaptive+ALS ($w=0.75$) & Locked & -0.007 [-0.019, 0.005] & -0.227 [-0.611, 0.127] & -0.007 [-0.011, -0.004] & 128 \\
OULAD & ALS & Locked & 0.000 [0.000, 0.000] & 0.000 [0.000, 0.000] & 0.000 [0.000, 0.000] & 128 \\
\bottomrule
\end{tabular}

\end{adjustbox}
\end{table}

\begin{table}[H]
\centering
\scriptsize
\caption{EdNet configured-update table with explicit anchor-weight endpoints on the 10k-user processed KT2 subset. Values are nominal-configuration minus reference deltas under matched candidate pools; the mixed pattern is a closed-loop stress test rather than a privacy--utility curve.}
\label{tab:rq2-dp-resilience-ednet}
\begin{adjustbox}{max width=\textwidth}
\begin{tabular}{llllllr}
\toprule
Dataset & Policy & Privacy Setting & $\Delta$ Corr. & $\Delta$ Post-test ($\times 10^{-3}$) & $\Delta$ Uptake & $n$ \\
\midrule
EdNet & Adaptive & Standard & 0.008 [-0.007, 0.026] & 0.391 [-0.066, 0.878] & 0.015 [0.010, 0.020] & 128 \\
EdNet & Adaptive+ALS ($w=0.50$) & Standard & -0.009 [-0.023, 0.003] & -0.712 [-1.171, -0.233] & -0.009 [-0.014, -0.004] & 128 \\
EdNet & Adaptive+ALS ($w=0.75$) & Standard & -0.013 [-0.024, -0.001] & -0.410 [-0.818, 0.021] & -0.011 [-0.016, -0.006] & 128 \\
EdNet & ALS & Standard & 0.000 [0.000, 0.000] & 0.000 [0.000, 0.000] & 0.000 [0.000, 0.000] & 128 \\
EdNet & Adaptive & Strong & 0.008 [-0.005, 0.021] & 0.103 [-0.356, 0.498] & 0.012 [0.008, 0.016] & 128 \\
EdNet & Adaptive+ALS ($w=0.50$) & Strong & -0.012 [-0.025, 0.001] & -0.794 [-1.173, -0.357] & -0.012 [-0.017, -0.007] & 128 \\
EdNet & Adaptive+ALS ($w=0.75$) & Strong & -0.013 [-0.025, -0.000] & -0.385 [-0.817, 0.078] & -0.010 [-0.017, -0.005] & 128 \\
EdNet & ALS & Strong & 0.000 [0.000, 0.000] & 0.000 [0.000, 0.000] & 0.000 [0.000, 0.000] & 128 \\
EdNet & Adaptive & Locked & 0.003 [-0.010, 0.015] & -0.163 [-0.614, 0.237] & 0.011 [0.007, 0.015] & 128 \\
EdNet & Adaptive+ALS ($w=0.50$) & Locked & -0.001 [-0.015, 0.012] & -0.528 [-0.905, -0.145] & -0.013 [-0.018, -0.008] & 128 \\
EdNet & Adaptive+ALS ($w=0.75$) & Locked & -0.015 [-0.027, -0.002] & -0.528 [-0.967, -0.080] & -0.009 [-0.014, -0.004] & 128 \\
EdNet & ALS & Locked & 0.000 [0.000, 0.000] & 0.000 [0.000, 0.000] & 0.000 [0.000, 0.000] & 128 \\
\bottomrule
\end{tabular}

\end{adjustbox}
\end{table}

\begin{table}[H]
\centering
\scriptsize
\caption{EdNet applied score-noise downstream stress test. Gaussian noise is injected into the adaptive score component and routed into actual slate selection; for \textsc{Adaptive+ALS}, the injected scale is multiplied by the adaptive blend weight, while the pure \textsc{ALS} endpoint receives no adaptive-component noise. Values are noisy $-$ clean deltas under matched candidate pools.}
\label{tab:ednet-applied-score-noise}
\begin{adjustbox}{max width=\textwidth}
\begin{tabular}{llllllllr}
\toprule
$\sigma_{\mathrm{score}}$ & $w$ & Policy & $\Delta$ Corr. & $\Delta$ Post-test ($\times 10^{-3}$) & $\Delta$ Uptake & Overlap & Kendall $\tau$ & $n$ \\
\midrule
0.02 & 0.00 & Adaptive & 0.007 & 0.173 & 0.010 & 0.809 & 0.904 & 128 \\
0.02 & 0.50 & Adaptive+ALS ($w=0.50$) & -0.004 & -0.294 & -0.006 & 0.944 & 0.989 & 128 \\
0.02 & 0.75 & Adaptive+ALS ($w=0.75$) & 0.005 & 0.351 & 0.009 & 0.913 & 0.993 & 128 \\
0.02 & 1.00 & ALS & 0.000 & 0.000 & 0.000 & 1.000 & 1.000 & 128 \\
0.05 & 0.00 & Adaptive & 0.008 & 0.093 & 0.002 & 0.827 & 0.870 & 128 \\
0.05 & 0.50 & Adaptive+ALS ($w=0.50$) & 0.000 & 0.061 & 0.004 & 0.954 & 0.972 & 128 \\
0.05 & 0.75 & Adaptive+ALS ($w=0.75$) & 0.009 & 0.485 & 0.025 & 0.916 & 0.981 & 128 \\
0.05 & 1.00 & ALS & 0.000 & 0.000 & 0.000 & 1.000 & 1.000 & 128 \\
0.10 & 0.00 & Adaptive & 0.018 & 0.683 & 0.017 & 0.812 & 0.843 & 128 \\
0.10 & 0.50 & Adaptive+ALS ($w=0.50$) & 0.007 & 0.178 & 0.009 & 0.949 & 0.945 & 128 \\
0.10 & 0.75 & Adaptive+ALS ($w=0.75$) & -0.006 & -0.152 & 0.008 & 0.909 & 0.964 & 128 \\
0.10 & 1.00 & ALS & 0.000 & 0.000 & 0.000 & 1.000 & 1.000 & 128 \\
0.20 & 0.00 & Adaptive & 0.008 & 0.120 & 0.008 & 0.813 & 0.748 & 128 \\
0.20 & 0.50 & Adaptive+ALS ($w=0.50$) & 0.002 & -0.079 & 0.005 & 0.932 & 0.899 & 128 \\
0.20 & 0.75 & Adaptive+ALS ($w=0.75$) & -0.003 & 0.004 & -0.004 & 0.902 & 0.928 & 128 \\
0.20 & 1.00 & ALS & 0.000 & 0.000 & 0.000 & 1.000 & 1.000 & 128 \\
\bottomrule
\end{tabular}

\end{adjustbox}
\end{table}

\begin{table}[H]
\centering
\scriptsize
\caption{EdNet 25k locked-tier follow-up. Values are DP $-$ non-private deltas under matched candidate pools for a larger processed KT2 subset. The result preserves the mixed EdNet pattern: locked DP does not induce a clear unanchored degradation regime, and anchor effects remain small in closed-loop metrics.}
\label{tab:ednet25k-locked-followup}
\begin{adjustbox}{max width=\textwidth}
\begin{tabular}{llllllr}
\toprule
Dataset & Policy & Privacy Setting & $\Delta$ Corr. & $\Delta$ Post-test ($\times 10^{-3}$) & $\Delta$ Uptake & $n$ \\
\midrule
EdNet-25k & Adaptive & Locked & 0.011 [-0.002, 0.024] & 0.002 [-0.524, 0.499] & 0.008 [0.001, 0.014] & 128 \\
EdNet-25k & Adaptive+ALS ($w=0.50$) & Locked & -0.001 [-0.013, 0.008] & 0.040 [-0.348, 0.469] & 0.004 [-0.001, 0.008] & 128 \\
EdNet-25k & Adaptive+ALS ($w=0.75$) & Locked & 0.003 [-0.009, 0.016] & 0.500 [0.097, 0.919] & 0.011 [0.006, 0.016] & 128 \\
EdNet-25k & ALS & Locked & 0.000 [0.000, 0.000] & 0.000 [0.000, 0.000] & 0.000 [0.000, 0.000] & 128 \\
\bottomrule
\end{tabular}

\end{adjustbox}
\end{table}

\begin{table}[H]
\centering
\scriptsize
\caption{High- and low-engagement subgroup gaps under the closed-loop simulator.}
\label{tab:subgroup-gap-app}
\begin{adjustbox}{max width=\textwidth}
\begin{tabular}{llll}
\toprule
Dataset & $\Delta$ Corr. (High$-$Low) & $\Delta$ Post-test (High$-$Low) ($\times 10^{-3}$) & $\Delta$ Uptake (High$-$Low) \\
\midrule
OULAD & 0.013 & -8.235 & 0.216 \\
EdNet & -0.006 & 3.196 & 0.031 \\
\bottomrule
\end{tabular}

\end{adjustbox}
\end{table}

\end{document}